\documentclass{article} \usepackage{iclr2018_conference,times}
\usepackage{hyperref}
\usepackage{url}
\usepackage{graphicx}
\usepackage{subcaption}
\usepackage{xcolor}
\usepackage{algorithm}
\usepackage{algpseudocode}
\usepackage{caption}
\usepackage{aliascnt}

\usepackage[T1]{fontenc}
\usepackage{inconsolata}

\usepackage[textsize=scriptsize,textwidth=3cm]{todonotes}
\makeatletter{}\usepackage{graphicx} \usepackage{subcaption} 
\usepackage{hyperref}
\usepackage{graphicx}
\usepackage{amsfonts}
\usepackage{amsmath}
\usepackage{amssymb}
\usepackage{amsthm}
\usepackage{xspace}
\usepackage{epstopdf}

\usepackage{multirow}
\usepackage{hhline}

\usepackage{mathtools}
\usepackage{breqn}

\newcommand{\calF}{\mathcal{F}}
\newcommand{\calN}{\mathcal{N}}
\newcommand{\calD}{\mathcal{D}}
\newcommand{\calM}{\mathcal{M}}
\newcommand{\Lap}[1]{\text{Lap}\left(#1\right)}

\newcommand{\eps}{\ensuremath{\varepsilon}}
\newcommand{\eqdef}{\ensuremath{\triangleq}}

\newcommand{\renyi}{R\'enyi\xspace}

\newcommand{\GM}[1]{\ensuremath{\calM_{#1}}}

\newcommand{\lead}[1]{\noindent\textbf{#1}}

\def\calR{\mathcal{R}}

\newtheorem{theorem}{Theorem}

\newaliascnt{definition}{theorem}
\newtheorem{definition}[definition]{Definition}
\aliascntresetthe{definition}

\newaliascnt{lemma}{theorem}
\newtheorem{lemma}[lemma]{Lemma}
\aliascntresetthe{lemma}

\newaliascnt{proposition}{theorem}
\newtheorem{proposition}[proposition]{Proposition}
\aliascntresetthe{proposition}

\newtheorem{corollary}[theorem]{Corollary}

\def\calA{\mathcal{A}}
\def\eps{\ensuremath{\varepsilon}}

\def\Pr{\mathbf{Pr}}

\def\bbR{\mathbb{R}}

\newcommand{\kset}[1]{[#1]}

\usepackage{sidecap}

\newcommand{\Div}[3]{D_{#1}\!\left(#2 \| #3\right)}

\DeclareMathOperator\erf{erf}
\DeclareMathOperator\erfc{erfc}
\DeclareMathOperator\erfcinv{erfc^{-1}}

\newcommand{\power}[2]{\left(#1\right)^{#2}}

\newcommand{\qt}{\ensuremath{\tilde{q}}}

\newcommand{\hi}{\mu}
\newcommand{\low}{\lambda}

\DeclareMathOperator*{\argmax}{\arg\!\max}

\newcommand{\expt}[2]{\mathbb{E}_{#1}\left[{#2}\right]}

\newcommand{\LS}{\mathrm{\tilde{LS}}}
\renewcommand{\SS}{\mathrm{SS}}

\newcommand{\order}{\ensuremath{\lambda}}
\newcommand{\MA}{\ensuremath{\boldsymbol{\beta}}}
\newcommand{\MAB}{\ensuremath{\boldsymbol{\beta}}}
\newcommand{\ma}[1]{\MA\left(#1\right)}
\newcommand{\emptyma}{\MA}
\newcommand{\dist}{\mathrm{dist}}
\newcommand{\barn}{\ensuremath{\bar{n}}}
\newcommand{\qu}{\mathrm{B_{U}}}
\newcommand{\ql}{\mathrm{B_{L}}}
\newcommand{\qswitch}{{q_0}}
\newcommand{\pdv}[2]{\frac{\partial #1}{\partial #2}}

\title{Scalable Private Learning with PATE}

\author{Nicolas Papernot\thanks{Equal contributions, authors ordered alphabetically. Work done while the authors were at Google Brain.}\\
Pennsylvania State University\\
\texttt{ngp5056@cse.psu.edu}\\
\And
Shuang Song$^*$\\
University of California San Diego\\
\texttt{shs037@eng.ucsd.edu}\\
\And
Ilya Mironov, Ananth Raghunathan, Kunal Talwar \& \'Ulfar Erlingsson \\
Google Brain\\
\texttt{\{mironov,pseudorandom,kunal,ulfar\}@google.com}
}

\iclrfinalcopy 
\begin{document}

\maketitle

\makeatletter{}\vspace*{-0.1in}

\begin{abstract}
The rapid adoption of machine learning
has increased concerns about
the privacy implications
of machine learning models trained on sensitive data,
such as medical records or other personal information.
To address those concerns,
one promising approach is
\emph{Private Aggregation of Teacher Ensembles}, or PATE,
which
transfers to a ``student'' model
the knowledge of
an ensemble of ``teacher'' models,
with intuitive privacy provided by training teachers on disjoint data
and strong privacy guaranteed by noisy aggregation of teachers' answers.
However,
PATE has so far been evaluated only
on simple classification tasks like MNIST,
leaving unclear
its utility 
when applied to larger-scale learning tasks and real-world datasets.

In this work, we show how PATE can
scale to learning tasks
with large numbers of output classes
and uncurated, imbalanced training data with errors.
For this, we introduce
new noisy aggregation mechanisms for teacher ensembles
that are more selective and add less noise,
and prove their tighter differential-privacy guarantees.
Our new mechanisms build on two insights:
the chance of teacher consensus
is increased by using more concentrated noise
and, lacking consensus, no answer need be given to a student.
The consensus answers used are more likely to be correct,
offer better intuitive privacy,
and incur lower-differential privacy cost.
Our evaluation shows
our mechanisms
improve on the original PATE on all measures,
and scale
to larger tasks
with both high utility
and very strong privacy ($\varepsilon<1.0$).

\end{abstract}

\makeatletter{}
\section{Introduction}
Many attractive applications of
modern machine-learning techniques
involve training models using highly sensitive data.
For example,
models trained on 
people's personal messages
or detailed medical information
can offer
invaluable insights
into real-world language usage
or the diagnoses and treatment of human diseases~\citep{mcmahan2017learning,liu2017detecting}.
A key challenge in such applications
is to prevent
models from revealing inappropriate details of the sensitive data---a
non-trivial task,
since models
are known to implicitly memorize such details during training
and also to inadvertently reveal them during inference~\citep{zhang2016understanding,shokri2016membership}.

Recently, 
two promising, new model-training approaches 
have offered the hope
that practical, high-utility machine learning
may be compatible with
strong privacy-protection guarantees for sensitive training data~\citep{pate-dpsgd}.
This paper revisits
one of these approaches,
\emph{Private Aggregation of Teacher Ensembles}, or PATE~\citep{papernot2016semi},
and develops techniques that
improve its scalability and practical applicability.
PATE
has the advantage of
being able to 
learn from the aggregated consensus 
of separate ``teacher'' models trained on disjoint data,
in a manner that both provides intuitive privacy guarantees
and is 
agnostic to the underlying machine-learning techniques
(cf.\ the approach of
differentially-private stochastic gradient descent \citep{abadi2016deep}).
In the PATE approach
multiple teachers are trained on disjoint sensitive data
(e.g., different users' data),
and uses the teachers' aggregate consensus answers in a black-box fashion
to supervise the training of
a ``student'' model.
By publishing only the student model (keeping the teachers private)
and
by adding carefully-calibrated Laplacian noise
to the aggregate answers used to train the student,
the original PATE work 
showed how to establish
rigorous $(\varepsilon,\delta)$ 
differential-privacy guarantees~\citep{papernot2016semi}---a gold standard of privacy~\citep{dwork2006calibrating}.
However, to date, PATE has been applied to 
only simple tasks, like MNIST,
without any realistic, larger-scale evaluation.

The techniques presented in this paper
allow PATE to be applied on a larger scale
to build more accurate models,
in a manner that 
improves
both on PATE's intuitive privacy-protection due to the teachers' independent consensus
as well as its differential-privacy guarantees.
As shown in our experiments,
the result is a gain in privacy, utility, and practicality---an
uncommon joint improvement.

\begin{figure}[t]
\centering
\begin{minipage}[b]{.33\textwidth}
	\centering
	\includegraphics[width=\textwidth]{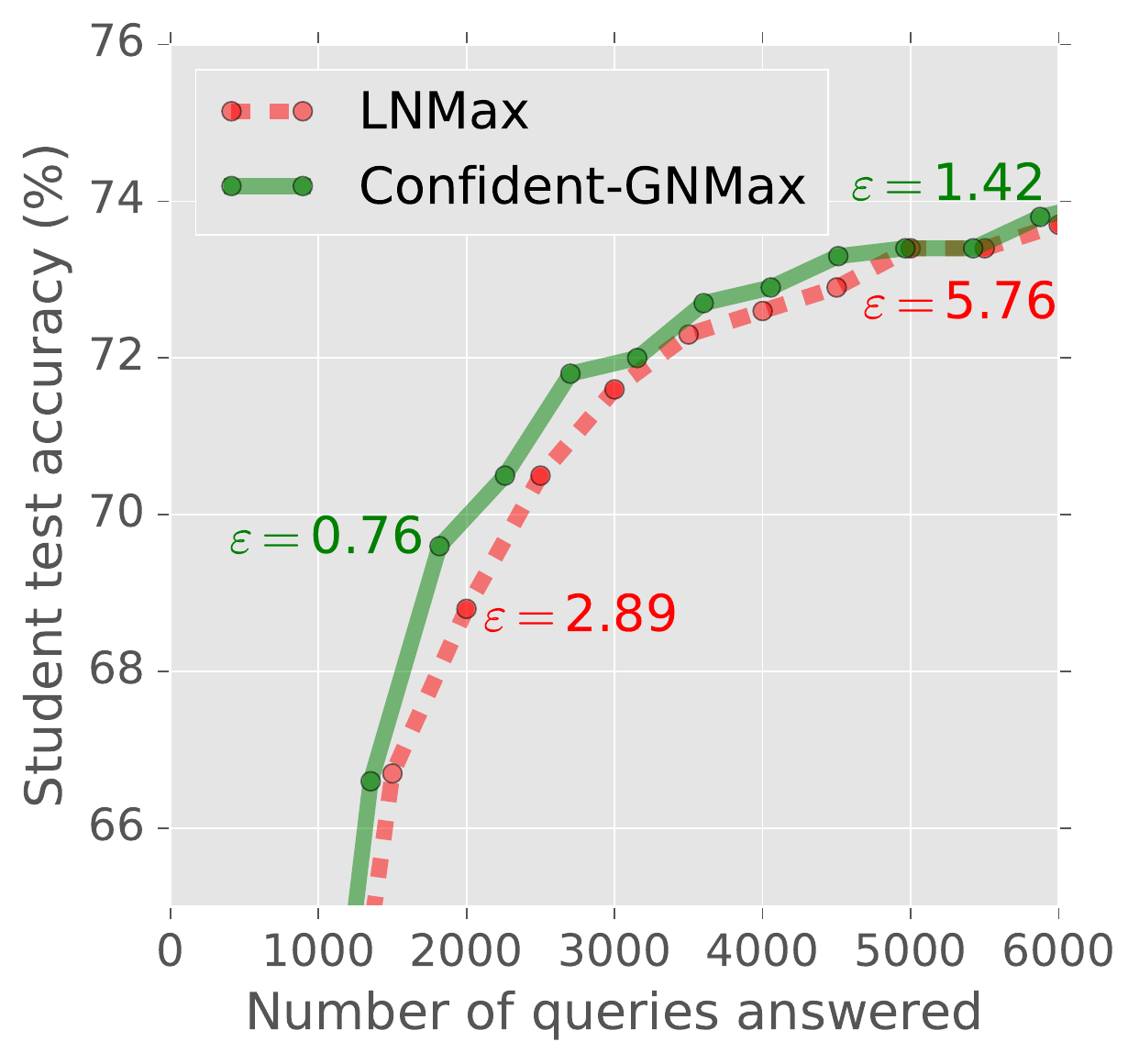}
\end{minipage}
\begin{minipage}[b]{.315\textwidth}
	\centering
	\includegraphics[width=\textwidth]{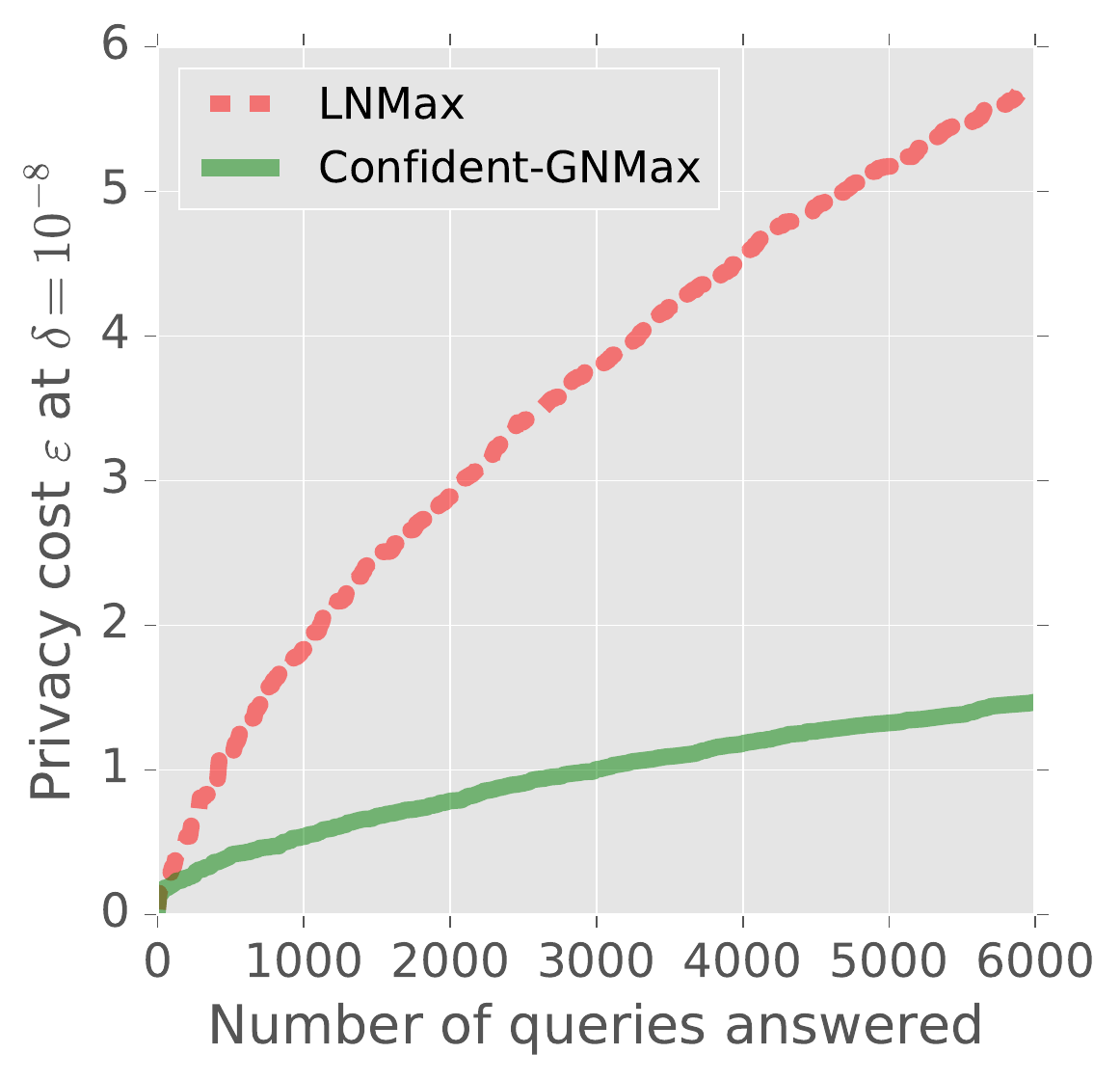}
\end{minipage}
\hspace*{.005\textwidth}
\begin{minipage}[b]{.33\textwidth}
	\centering
	\includegraphics[width=\textwidth]{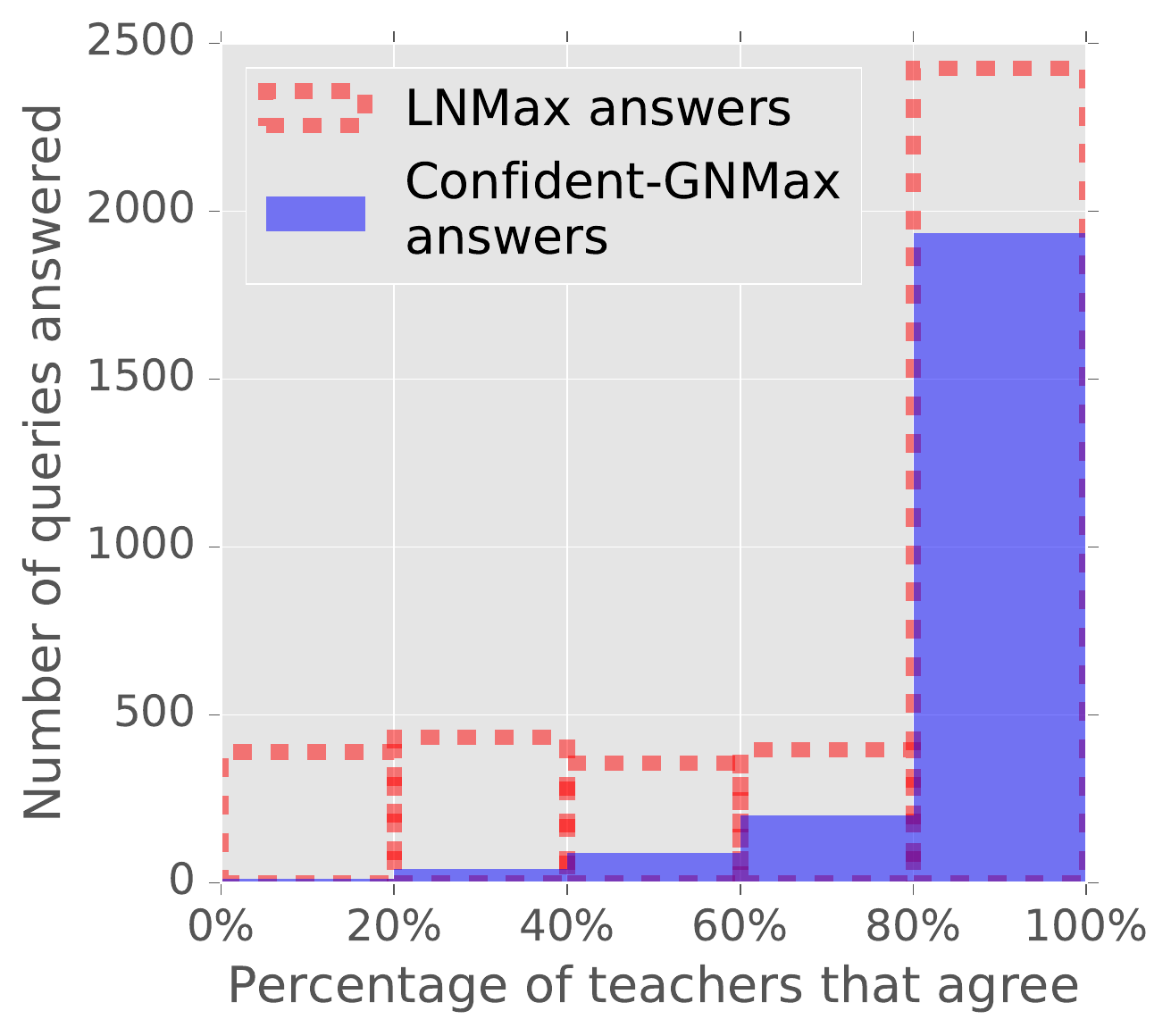}
\end{minipage}
\caption{
  Our contributions are techniques (Confident-GNMax)
  that improve on the original PATE (LNMax)
  on all measures.
  \textit{Left:} Accuracy is higher throughout training,
  despite greatly improved privacy (more in \autoref{table:results_summary}).
  \textit{Middle:} The $\varepsilon$ differential-privacy bound on privacy cost
  is quartered, at least (more in \autoref{fig:threshold-check}).
  \textit{Right:} Intuitive privacy is also improved,
  since students are trained on answers with a much stronger consensus among the teachers
  (more in \autoref{fig:threshold-check}).
  These are results 
  for a character-recognition task,
  using the most favorable LNMax parameters
  for a fair comparison.
}
\label{fig:aggreg-lap-vs-gauss}
\end{figure}

The primary technical contributions of this paper
are new mechanisms for aggregating
teachers' answers
that are more selective and add less noise.
On all measures,
our techniques improve on the original PATE mechanism
when evaluated on the same tasks using the same datasets,
as described in \autoref{sec:expt-eval}.
Furthermore,
we evaluate both variants of PATE on a
new, large-scale character recognition task
with 150 output classes, inspired by MNIST.
The results show that PATE can be successfully utilized
even to uncurated datasets---with
significant class imbalance
as well as erroneous class labels---and
that our new aggregation mechanisms
improve both privacy and model accuracy.

To be more selective,
our new mechanisms 
leverage some pleasant synergies between privacy and utility
in PATE aggregation.
For example,
when teachers disagree,
and there is no real consensus,
the privacy cost is much higher;
however,
since such disagreement
also suggest that the teachers may not give a correct answer,
the answer may simply be omitted.
Similarly,
teachers may avoid giving an answer
where the student already is confidently predicting the right answer.
Additionally, we ensure that these selection steps are themselves done
in a private manner.

To add less noise,
our new PATE aggregation mechanisms
sample Gaussian noise,
since the tails of that distribution diminish far more rapidly
than those of the Laplacian noise
used in the original PATE work.
This reduction greatly increases the chance
that the noisy aggregation of teachers' votes
results in the correct consensus answer,
which is especially important when PATE is scaled
to learning tasks with large numbers of output classes.
However,
changing the sampled noise
requires redoing the entire PATE privacy analysis from scratch
(see \autoref{sec:gaussian-pate} and details in \autoref{ap:privacy-analysis}).

Finally,
of independent interest
are the details of
our evaluation
extending that of the original PATE work.
In particular, 
we find 
that the virtual adversarial training (VAT) technique of~\citet{miyato2017virtual}
is a good 
basis for semi-supervised learning on tasks with many classes,
outperforming
the improved GANs by~\cite{salimans2016improved}
used in the original PATE work.
Furthermore, 
we explain how to tune
the PATE approach
to achieve very strong privacy ($\varepsilon \approx 1.0$)
along with high utility,
for our
real-world character recognition learning task.

This paper is structured as follows:
\autoref{sec:related-work} is the related work section;
\autoref{sec:pate-overview} gives a background on PATE and an overview of our
work;
\autoref{sec:gaussian-pate} describes our improved aggregation mechanisms;
\autoref{sec:expt-eval} details our experimental evaluation;
\autoref{sec:conclude} offers conclusions;
and proofs are deferred to the Appendices.
 
\makeatletter{}\section{Related Work}
\label{sec:related-work}

Differential privacy is by now the gold standard of privacy. It offers a rigorous framework whose threat model
makes few assumptions about the adversary's capabilities, allowing differentially
private algorithms to effectively cope against strong adversaries.
This is not the case of all privacy definitions, as demonstrated by
successful attacks against anonymization techniques~\citep{aggarwal2005k,narayanan2008robust,bindschaedler2017plausible}.

The first learning algorithms adapted to provide differential privacy
with respect to their training data were often linear and convex~\citep{pathak2010multiparty,chaudhuri2011differentially,song2013stochastic,bassily2014differentially,hamm2016learning}.
More recently, successful developments in deep learning called for
differentially private stochastic gradient descent algorithms~\citep{abadi2016deep},
some of which have been tailored to learn in 
federated~\citep{mcmahan2017learning} settings.

Differentially private selection mechanisms like GNMax (\autoref{sec:max-of-gaussian}) are commonly used in 
hypothesis testing, frequent itemset mining, and as building blocks of more
complicated private mechanisms. The most commonly used differentially private
selection mechanisms are exponential mechanism~\citep{mcsherry2007mechanism}
and LNMax~\citep{bhaskar2010discovering}. Recent works offer lower bounds on sample complexity
of such problem~\citep{steinke2017tight,bafna-ullman}.

The Confident and Interactive Aggregator proposed in our work
(\autoref{ssec:confident-aggregator} and \autoref{ssec:interactive-protocol}
resp.) use the intuition
that selecting samples under certain constraints could result in better training
than using samples uniformly at random.
In Machine Learning Theory, active learning~\citep{cohn1994improving} has been shown to allow learning from fewer labeled examples than the passive case (see e.g. \citet{hanneke2014theory}). Similarly, in model stealing \citep{tramer2016stealing}, a goal is to learn a model from limited access to a teacher network.
There is previous work in differential privacy literature~\citep{hardt2010multiplicative, roth2010interactive} where the
mechanism first {\em decides} whether or not to answer a query, and then privately answers the
queries it chooses to answer using a traditional noise-addition mechanism. In these cases,
the sparse vector technique~\citep[Chapter 3.6]{dwork2014algorithmic} helps bound the privacy cost in terms of
the number of answered queries. This is in contrast to our work where a constant \emph{fraction} of
queries get answered and the sparse vector technique does not seem to help
reduce the privacy cost. Closer to our work, \citet{bun2017make} consider a
setting where the answer to a query of interest is often either very large or
very small. They show that a sparse vector-like analysis applies in this case,
where one pays only for queries that are in the middle.
 
\makeatletter{}\section{Background and Overview}
\label{sec:pate-overview}
We introduce essential components of our approach towards a generic and 
flexible framework for machine 
learning with provable privacy guarantees for training data. 

\subsection{The PATE Framework}
\label{sec:pate-background}

Here, we provide an overview of the PATE framework.
To protect the privacy of training data during learning, PATE transfers knowledge
from an ensemble of teacher models trained on partitions
of the data to a student model.
Privacy guarantees may be understood intuitively and 
expressed rigorously in terms of differential privacy.

Illustrated in \autoref{fig:approach-overview},
the PATE framework
consists of three key parts: (1) an ensemble of $n$ teacher
models, (2) an aggregation mechanism and (3) a student model. 

\begin{figure}[h]
  \centering
  \includegraphics[width=\columnwidth]{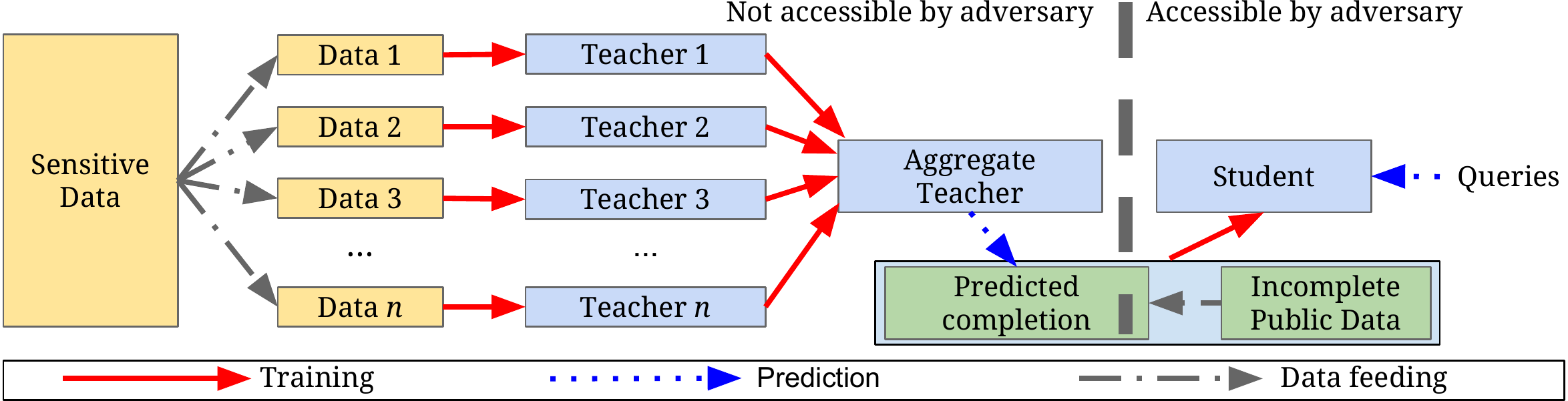}
  \caption{Overview of the approach: (1) an ensemble of teachers 
  is trained on disjoint subsets of the sensitive data, (2) a 
  student model is trained on public data labeled using the ensemble.}
  \label{fig:approach-overview}
\end{figure}

\noindent
\textbf{Teacher models:} 
 Each teacher is a model trained
independently on a subset of the data whose privacy one wishes
to protect. 
The data is partitioned to ensure no
pair of teachers will have trained on overlapping data. 
Any learning technique
suitable for the data can be used for any teacher.
Training each teacher on a \emph{partition} of the
sensitive data produces $n$ different models solving the same task. At
inference, teachers independently predict labels.

\noindent
\textbf{Aggregation mechanism:} 
When there is a strong consensus among teachers, 
the label they almost all agree on does not depend on the model
learned by any given teacher.
Hence, this collective decision is intuitively private with respect to any
given training point---because such a point could have been included
only in one of the teachers' training set.
To provide rigorous guarantees of differential privacy, the aggregation mechanism of the original PATE
framework counts votes assigned to each class, adds carefully calibrated
Laplacian noise to the resulting vote histogram, and outputs the class with
the most noisy votes as the ensemble's prediction. This mechanism is referred to as
the max-of-Laplacian mechanism, or LNMax, going forward.

For samples $x$ and classes $1, \ldots, m$,  let $f_j(x) \in \kset{m}$ denote
the $j$-th teacher model's prediction and $n_i$ denote the vote count for the
$i$-th class (i.e., $n_i \eqdef |f_j(x) = i|$). The output of the mechanism is
$\calA(x) \eqdef \argmax_i \left(n_i(x) + \Lap{1/\gamma}\right)$. Through a
rigorous analysis of this mechanism, the PATE framework provides a
differentially private API: the privacy cost of each aggregated prediction
made by the teacher ensemble is known.

\noindent
\textbf{Student model:} PATE's final step involves the training of a student
model by knowledge transfer from the teacher ensemble using access to
public---but \emph{unlabeled}---data. To limit the privacy cost of labeling
them, queries are only made to the aggregation mechanism for a subset of
public data to train the student in a semi-supervised way using a fixed number
of queries. The authors note that every additional ensemble prediction
increases the privacy cost spent and thus cannot work with unbounded queries.
Fixed queries fixes privacy costs as well as diminishes the value of attacks
analyzing model parameters to recover training
data~\citep{zhang2016understanding}. The student only sees public data and
privacy-preserving labels.

\subsection{Differential Privacy}

Differential privacy~\citep{dwork2006calibrating} 
requires that
the sensitivity of the distribution of an algorithm's output to small
perturbations of its input be limited. The following variant of the definition 
captures this intuition formally: 
\begin{definition} 
A randomized mechanism $\calM$ with domain $\calD$ and range
$\mathcal{R}$ satisfies $(\eps,\delta)$-differential privacy if for
any two adjacent inputs $D, D'\in \calD$ and for any subset of
outputs $S\subseteq\calR$ it holds that: 
	\begin{equation}
	\label{eq:dp}
	\Pr[\calM(D)\in S]\leq e^{\eps} \cdot \Pr[\calM(D')\in S]+\delta.
	\end{equation}
\end{definition}

For our application of differential privacy to ML, 
adjacent inputs are defined as two datasets that only differ by one training
example and the randomized mechanism $\calM$ would be the model training
algorithm.
The privacy parameters have the following natural interpretation: $\eps$ is an upper bound on the loss of privacy, and $\delta$ is the probability with which this guarantee may not hold. Composition theorems~\citep{dwork2014algorithmic} allow us to keep track of the privacy cost when we run a sequence of mechanisms.

\subsection{\renyi Differential Privacy}
\label{sec:rdp}

\cite{papernot2016semi} note that the natural approach to bounding PATE's privacy
loss---by bounding the privacy cost of each label queried and
using strong composition~\citep{dwork2010boosting} to derive the total
cost---yields loose privacy guarantees. Instead, their approach
uses \emph{data-dependent} privacy analysis. This
takes advantage of the fact that when the consensus among the teachers is very
strong, the plurality outcome has overwhelming likelihood leading to a very
small privacy cost whenever the consensus occurs. To capture this effect quantitatively,
\cite{papernot2016semi} rely on the \emph{moments accountant}, introduced by~\cite{abadi2016deep}
and building on previous work~\citep{bun2016concentrated, dwork2016concentrated}.

In this section, we recall the language of
\renyi Differential Privacy or RDP \citep{mironov2017renyi}. RDP generalizes
pure differential privacy ($\delta = 0$) and is closely related to the moments
accountant. We choose to use RDP as a more natural analysis framework when dealing with
our mechanisms that use Gaussian noise. Defined below, the RDP of a mechanism
is stated in terms of the \renyi divergence. \medskip

\begin{definition}[\renyi Divergence]
The \renyi divergence of order $\low$ between two distributions $P$ and~$Q$
  is defined as:
\[\Div{\low}{P}{Q}
\eqdef \frac{1}{\low - 1} \log \expt{x\sim Q}{\power{P(x)/Q(x)}\low}
= \frac{1}{\low - 1} \log \expt{x\sim P}{\power{P(x)/Q(x)}{\low-1}}.
\]
\end{definition}\medskip

\begin{definition}[\renyi Differential Privacy (RDP)]
A randomized mechanism $\calM$ is said to guarantee $(\order,\eps)$-RDP with $\order \geq 1$ if for any neighboring datasets $D$ and $D'$,
\begin{align*}
\Div{\low}{\calM(D)}{\calM(D')} 
  = \frac{1}{\low - 1} \log \expt{x\sim \calM(D)}{\power{\frac{\Pr\left[\calM(D)=x
  \right]\hfill}{\Pr\left[\calM(D')=x\right]}}{\low-1}}
\leq \eps.
\end{align*}
\end{definition}

RDP generalizes pure differential privacy in the sense that $\eps$-differential privacy is equivalent to $(\infty, \eps)$-RDP. \cite{mironov2017renyi} proves the following key facts
that allow easy composition of RDP guarantees and their conversion to $(\eps, \delta)$-differential privacy bounds.

\begin{theorem}[Composition]\label{thm:ma_composition}
If a mechanism $\calM$ consists of a sequence of adaptive mechanisms $\calM_1, \dots, \calM_k$ such that for any $i\in[k]$, $\calM_i$ guarantees $(\low, \eps_i)$-RDP, then $\calM$ guarantees $(\low, \sum_{i=1}^k \eps_i)$-RDP.
\end{theorem}\smallskip

\begin{theorem}[From RDP to DP]\label{thm:ma_convert}
If a mechanism $\calM$ guarantees $(\low, \eps)$-RDP, then
$\calM$ guarantees $(\eps + \frac{\log 1/\delta}{\low - 1}, \delta)$-differential privacy 
for any $\delta \in (0, 1)$.
\end{theorem}

While both $(\eps,\delta)$-differential privacy and RDP are relaxations of
pure $\eps$-differential privacy,
the two main advantages of RDP are as follows. First, it composes nicely;
second, it captures the privacy guarantee of Gaussian noise in a much cleaner
manner compared to $(\eps,\delta)$-differential privacy. This lets us do a
careful privacy analysis of the GNMax mechanism as stated in
\autoref{theorem:higher-to-lower}. While the analysis of
\cite{papernot2016semi} leverages the first aspect of such frameworks with the
Laplace noise (LNMax mechanism), our analysis of the GNMax mechanism relies on both.

\subsection{PATE Aggregation Mechanisms}
\label{ssec:aggregation}

The aggregation step is a crucial component of PATE. It
enables knowledge transfer from the teachers
to the student while enforcing privacy. We improve the LNMax mechanism
used by~\cite{papernot2016semi} which adds Laplace noise to 
teacher votes and outputs the class with the highest votes. 

First, we
add Gaussian noise with an accompanying privacy analysis in the RDP framework.
This modification effectively reduces the noise needed to achieve the same privacy cost per student query.

Second, the aggregation mechanism is now \emph{selective}: 
teacher votes are analyzed to decide which student queries are \emph{worth} answering. This
takes into account both the privacy cost of each query and its payout in
improving the student's utility. Surprisingly, our analysis shows that these two
metrics are not at odds and in fact align with each other: the privacy cost is
the smallest when teachers agree, and when teachers agree, the label is more
likely to be correct thus being more useful to the student.

Third, we propose and study an \emph{interactive} mechanism that takes into
account not only teacher votes on a queried example but possible student
predictions on that query. Now, queries worth answering are those where the
teachers agree on a class but the student is not confident in its prediction
on that class. This third modification aligns the two metrics discussed above even further:
queries where the student already agrees with the consensus of teachers are
not worth expending our privacy budget on, but queries 
where the student is
less confident are useful and answered at a small privacy cost.

\subsection{Data-dependent Privacy in PATE}

A direct privacy analysis of the aggregation mechanism, for reasonable values of
the noise parameter, allows answering only few queries
before the privacy cost becomes prohibitive. The original PATE proposal used a
data-dependent analysis, exploiting the fact that when the teachers have
large agreement, the privacy cost is usually much smaller than the
data-independent bound would suggest.

In our work, we perform a data-dependent privacy analysis of the aggregation
mechanism with Gaussian noise. This change of noise distribution turns out be technically much more
challenging than the Laplace noise case and we defer the details to
\autoref{ap:privacy-analysis}. This increased complexity of the analysis
however does not make the algorithm any more complicated and thus
allows us to improve the privacy-utility tradeoff.

\paragraph{Sanitizing the privacy cost via smooth sensitivity analysis.} An
additional challenge with data-dependent privacy analyses arises from the fact
that the privacy cost itself is now a function of the private data.  Further,
the data-dependent bound on the privacy cost has large global sensitivity (a
metric used in differential privacy to calibrate the noise injected) and is
therefore difficult to sanitize. To remedy this, we use the smooth sensitivity
framework proposed by \citet{nissim2007smooth}.

\autoref{ap:smooth-sensitivity} describes how we add noise to the computed
privacy cost using this framework to publish a sanitized version of the
privacy cost. \autoref{sec:computing-smooth-sensitivity} defines
smooth sensitivity and outlines algorithms \ref{alg:ls}--\ref{alg:ss} that
compute it. The rest of \autoref{ap:smooth-sensitivity} argues the
correctness of these algorithms. The final analysis shows that the incremental
cost of sanitizing our privacy estimates is modest---less than 50\% of the raw estimates---thus
enabling us to use precise data-dependent privacy analysis
while taking into account its privacy implications.

\makeatletter{}\section{Improved Aggregation Mechanisms for PATE}
\label{sec:gaussian-pate}
The privacy guarantees provided by PATE stem from the design and
analysis of the aggregation step.
Here, we detail our improvements to the
mechanism used by~\cite{papernot2016semi}. As outlined in \autoref{ssec:aggregation}, we first replace the Laplace
noise added to teacher votes with Gaussian noise, adapting the
data-dependent privacy analysis. Next, we describe the Confident and
Interactive Aggregators that select 
queries worth answering in a privacy-preserving way: the privacy budget is shared between the query selection and answer computation. The aggregators use different heuristics to select queries: the former does not take into account student
predictions, while the latter does.

\subsection{The GNMax Aggregator and Its Privacy Guarantee}
\label{sec:max-of-gaussian}
This section uses the following notation. For a sample $x$ and classes $1$ to $m$,  
let $f_j(x) \in \kset{m}$ denote the $j$-th teacher model's
prediction on $x$ and $n_i(x)$ denote the vote count for the $i$-th class (i.e., $n_i(x) =
|\{j\colon f_j(x) = i\}|$). We define a Gaussian NoisyMax (GNMax) aggregation mechanism as:
\[
  \GM{\sigma}(x) \eqdef \argmax_i \left \{ n_i(x) + \calN(0,\sigma^2)\right \},
\]
where $\calN(0, \sigma^2)$ is the Gaussian distribution
with mean 0 and variance $\sigma^2$. The aggregator outputs the class with
noisy plurality after adding Gaussian noise to each vote count. In what
follow, \emph{plurality} more generally refers to the
highest number of teacher votes assigned among the classes.

The Gaussian distribution is more concentrated than the Laplace distribution 
used by~\cite{papernot2016semi}. This concentration directly improves the aggregation's utility 
when the number of classes~$m$ is large.
The GNMax mechanism satisfies $(\order,\order/\sigma^2)$-RDP, which
holds for all inputs and all $\order\geq 1$ (precise statements and proofs of
claims in this section are deferred to \autoref{ap:privacy-analysis}). A straightforward application of
composition theorems leads to loose privacy bounds.
As an example, the standard advanced composition theorem applied to
experiments in the last two rows of \autoref{table:results_summary} would give us $\eps=8.42$
and $\eps=10.14$ resp.~at $\delta=10^{-8}$ for the Glyph dataset.

To refine these, we work out a careful \emph{data-dependent} analysis
that yields  values of $\eps$ smaller than~$1$ for the same~$\delta$.
The following theorem translates data-independent RDP guarantees for higher orders 
into a data-dependent RDP guarantee for a smaller order $\order$. We use it in conjunction
with \autoref{prop:gaussian-q} to bound the privacy cost of each query to the GNMax algorithm
as a function of $\qt$, the probability that the most common answer will not be output by the mechanism.\medskip

\edef\thmtransfer{\the\value{theorem}}
\begin{theorem}[informal]\label{theorem:higher-to-lower} Let $\calM$ be a randomized algorithm with $\left(\hi_1,
\eps_1\right)$-RDP and $\left(\hi_2, \eps_2\right)$-RDP guarantees
and suppose that given a dataset $D$, there exists \emph{a likely} outcome $i^*$ such that $\Pr\left[\calM(D) \neq i^* \right] \leq \qt$. Then the \emph{data-dependent} \renyi differential privacy for $\calM$ of order $\order\leq \hi_1, \hi_2$ at $D$ is bounded by a function of $\qt$, $\hi_1$, $\eps_1$, $\hi_2$, $\eps_2$, which approaches 0 as $\qt\rightarrow 0$.
\end{theorem}

The new bound improves on the data-independent privacy for $\order$ as long as the distribution of the algorithm's output \emph{on that input} has a strong peak (i.e., $\qt\ll 1$). 
Values of $\qt$ close to $1$ could result in a looser bound. Therefore, in
practice we take the minimum between this bound and $\order/\sigma^2$ (the
data-independent one).
The theorem generalizes Theorem 3 from~\cite{papernot2016semi}, where it was shown for a mechanism satisfying $\eps$-differential
privacy (i.e., $\hi_1=\hi_2=\infty$ and $\eps_1=\eps_2$).

The final step in our analysis uses the following lemma to bound the probability $\qt$ when 
$i^*$ corresponds to the class with the true plurality of teacher votes.

\edef\lemmaqt{\the\value{theorem}}
\begin{proposition}\label{prop:gaussian-q}
  For any $i^* \in \kset{m}$, we have $\Pr\left[\GM{\sigma}(D) \neq i^*\right] \leq
  \frac{1}{2} \sum_{i\neq i^*} \mathrm{erfc}\left(\frac{n_{i^*}-n_i}{2\sigma}\right),$
where $\mathrm{erfc}$ is the complementary error function.
\end{proposition}

In \autoref{ap:privacy-analysis}, we detail how these results translate to
privacy bounds. In short, for each query to the GNMax aggregator,
given teacher votes $n_i$ and the class $i^*$ with maximal support, \autoref{prop:gaussian-q} gives us the value of $\qt$ to use in \autoref{theorem:higher-to-lower}. We optimize over $\hi_1$ and $\hi_2$ to
get a data-dependent RDP guarantee for any order $\order$. Finally, we use
composition properties of RDP to analyze a sequence of queries, and translate
the RDP bound back to an $(\eps,\delta)$-DP bound.

\paragraph{Expensive queries.} This data-dependent privacy analysis leads us
to the concept of an \emph{expensive} query in terms of its privacy
cost. When teacher votes largely
disagree, some $n_{i^*} - n_i$ values may be small leading to a large value for
$\qt$: i.e., the lack of consensus
amongst teachers indicates that the aggregator is likely to output a wrong
label. 
Thus expensive queries from a privacy perspective are often bad for
training too. Conversely, queries with strong consensus 
enable tight privacy bounds. This synergy motivates the aggregation mechanisms
discussed in the following sections: they evaluate the strength of the
consensus before answering a query. 
\subsection{The Confident-GNMax Aggregator}
\label{ssec:confident-aggregator}

In this section, we
propose a refinement of the GNMax aggregator that
enables us to filter out queries for which teachers do not have a sufficiently
strong consensus. This filtering enables the teachers to avoid answering 
expensive queries.
We also take note to do this selection step itself in a private
manner. 

The proposed \emph{Confident Aggregator} is described in
\autoref{alg:confident aggregator}.
To select queries with overwhelming consensus, the algorithm checks if
the plurality vote crosses a threshold $T$. To enforce privacy in this step,
the comparison is done after adding Gaussian noise with variance $\sigma_1^2$.
Then, for queries that pass this noisy threshold check, the aggregator
proceeds with the usual GNMax mechanism with a smaller variance $\sigma_2^2$.
For queries that do not pass the noisy threshold check, the aggregator
simply returns $\bot$ and the student discards this example in its training.

In practice, we often choose significantly higher values for $\sigma_1$ compared to
$\sigma_2$. This is because we pay the cost of the noisy threshold check \emph{always}, and
without the benefit of knowing that the consensus is strong. We pick $T$ so that queries where the plurality gets less than half
the votes (often very expensive) are unlikely to pass the threshold
after adding noise, but we still have a high enough yield amongst the queries with a
strong consensus. This tradeoff leads us to look for $T$'s between $0.6 \times$ to $0.8
\times$ the number of teachers.

The privacy cost of this aggregator is intuitive: we pay for the
threshold check
for every query, and for the GNMax step only for queries that pass
the check.
In the work of \cite{papernot2016semi}, the mechanism paid a privacy cost for every
query, expensive or otherwise. In comparison, the Confident Aggregator expends a much smaller
privacy cost to check against the threshold, and by answering a significantly
smaller fraction of expensive queries, it
expends a lower privacy cost overall.

\begin{algorithm}[t] \caption{\textbf{-- Confident-GNMax Aggregator:} given a query, consensus among teachers is first estimated in a privacy-preserving way to then only reveal confident teacher predictions.}
\label{alg:confident aggregator}
\begin{algorithmic}[1] \Require input $x$, threshold $T$, noise parameters $\sigma_1$ and $\sigma_2$
\If{$\max_i\{n_j(x)\} + \mathcal{N}(0, \sigma_1^2) \geq T$} \Comment{Privately
  check for consensus}
	\State\Return $\argmax_j \left\{ n_j(x) + \mathcal{N}(0, \sigma_2^2)\right\}$
	\Comment{Run the usual max-of-Gaussian}
\Else
	\State\Return $\bot$
\EndIf
\end{algorithmic}
\end{algorithm}

\subsection{The Interactive-GNMax Aggregator}
\label{ssec:interactive-protocol}

While the Confident Aggregator excludes expensive queries, it
ignores the possibility that the student might receive labels that contribute
little to learning, and in turn to its utility. By incorporating the student's
current predictions for its public training data, we design an \emph{Interactive
Aggregator} that discards queries where the student already
confidently predicts the same label as the teachers. 

Given a set of queries,
the Interactive Aggregator (\autoref{alg:interactive-protocol}) 
selects those answered by
comparing student predictions to teacher votes for each class. Similar to Step 1 in
the Confident Aggregator, queries where the plurality of these noised differences
crosses a threshold are answered with GNMax.
This noisy threshold suffices to enforce privacy of the first step
because student predictions can be considered public information 
(the student is trained in a differentially private manner).

For queries that fail this check, the mechanism reinforces the predicted
student label if the student is confident enough and does this without
looking at teacher votes again. This limited form of supervision comes at a small
privacy cost.
Moreover, the order of the checks ensures that a student falsely
confident in its predictions on a query is not accidentally reinforced if it
disagrees with the teacher consensus. The privacy accounting is 
identical to the Confident Aggregator except in considering the difference between
teachers and the student instead of only the teachers votes.

In practice, the Confident
Aggregator can be used to start training a student when it can make no
meaningful predictions and training can be finished off with the Interactive
Aggregator after the student gains some proficiency.

\begin{algorithm}[t]
\caption{\textbf{-- Interactive-GNMax Aggregator}: the protocol first compares student predictions to the teacher votes in a privacy-preserving way to then either (a) reinforce the student prediction for the given query or (b) provide the student with a new label predicted by the teachers.}
\label{alg:interactive-protocol}
\begin{algorithmic}[1]
\Require input $x$, confidence $\gamma$, threshold $T$, noise parameters
  $\sigma_1$ and $\sigma_2$, total number of teachers $M$
\State{Ask the student to provide prediction scores $\textbf{p}(x)$}
\If{$\max_j \{n_j(x) - M p_j(x)\} + \mathcal{N}(0, \sigma_1^2) \geq T$} \Comment{Student does not agree with teachers}
	\State\Return $\argmax_j \{n_j(x) + \mathcal{N}(0, \sigma_2^2)\}$ \Comment{Teachers provide new label}
\ElsIf{$\max\{p_{i}(x)\} > \gamma$} \Comment{Student agrees with teachers and is confident}
	\State\Return $\arg\max_j p_j(x)$ \Comment{Reinforce student's prediction}
\Else
	\State\Return $\bot$ \Comment{No output given for this label}
\EndIf
\end{algorithmic}
\end {algorithm}

\makeatletter{}\section{Experimental Evaluation} \label{sec:expt-eval}

Our goal is first to show that the improved aggregators
introduced in \autoref{sec:gaussian-pate}
enable the application of PATE to uncurated data, thus
departing from previous results on tasks with balanced and well-separated
classes. We
experiment with the Glyph dataset described below to 
address two aspects left open by~\cite{papernot2016semi}: (a) the
performance of PATE on a task with a larger number of
classes (the framework was only evaluated on datasets with at most 10 classes) and (b)
the privacy-utility tradeoffs offered by PATE on
data that is class imbalanced and partly mislabeled. In \autoref{ssec:laplace-vs-gaussian}, we evaluate the
improvements given by the GNMax aggregator
over its Laplace counterpart (LNMax) and demonstrate the necessity of the
Gaussian mechanism for uncurated tasks.

In \autoref{ssec:student-training}, we then evaluate the performance of PATE with both the Confident and Interactive
Aggregators on all datasets used to benchmark the original PATE framework,
in addition to Glyph. 
With the right teacher and
student training, the two mechanisms from \autoref{sec:gaussian-pate}
achieve high accuracy
with very tight privacy bounds. Not answering queries
for which teacher consensus is too low (Confident-GNMax) or
the student's predictions already agree with teacher votes (Interactive-GNMax)
better aligns utility and privacy: queries are answered at a significantly reduced cost.

\subsection{Experimental Setup}
\label{sec:expt-setup}

\paragraph{MNIST, SVHN, and the UCI Adult databases.} We evaluate with two computer vision
tasks (MNIST and Street View House Numbers~\citep{netzer2011reading}) and
census data from the UCI Adult dataset~\citep{kohavi1996scaling}.
This enables a comparative analysis of the utility-privacy
tradeoff achieved with our Confident-GNMax aggregator and the LNMax originally used 
in PATE. We replicate the experimental setup and results found 
in~\cite{papernot2016semi} with code and teacher votes made 
available online. The source code for the privacy analysis in this paper as
well as supporting data required to run this analysis is available on
Github.\footnote{\url{https://github.com/tensorflow/models/tree/master/research/differential_privacy}}

A detailed description of the experimental 
setup can be found in~\cite{papernot2016semi};
we provide here only a brief overview.
For MNIST and SVHN, teachers are convolutional networks trained on partitions
of the training set. For UCI Adult, each teacher is a random forest.
The test set is split in two halves: the first is used as unlabeled
inputs to simulate the student's public data
and the second is used as a hold out to evaluate test performance.
The MNIST and SVHN students are convolutional networks trained using semi-supervised learning with GANs
 \`a la~\cite{salimans2016improved}. The student for the Adult dataset are fully supervised random forests.

\paragraph{Glyph.} This optical character recognition task has \emph{an order of
magnitude more classes} than all previous applications of PATE.
The Glyph dataset also possesses many characteristics shared by real-world
tasks: e.g., it is imbalanced and some inputs are mislabeled.
Each input is
a $28 \times 28$ grayscale image containing a single glyph
generated synthetically from a collection of over 500K computer fonts.\footnote{Glyph data
is not public but similar data is available publicly as part of the notMNIST dataset.}
Samples representative of the difficulties raised by the data are depicted
in \autoref{fig:glyph-samples}.
The task is to classify inputs as one of the $150$ Unicode
symbols used to generate them. 

This set of 150 classes results from pre-processing
efforts. We discarded additional classes that had few samples; some classes
had at least 50 times fewer inputs than the most popular classes, and 
these were almost exclusively incorrectly labeled inputs. 
We also merged classes that were too
ambiguous for even a human to differentiate them.
Nevertheless, a manual inspection of samples grouped by classes---favorably to
the human observer---led to the conservative estimate that
some classes remain 5 times more frequent, and mislabeled inputs 
represent at least $10\%$ of the data.

\begin{figure}[p]
	\centering
	\includegraphics[width=\textwidth]{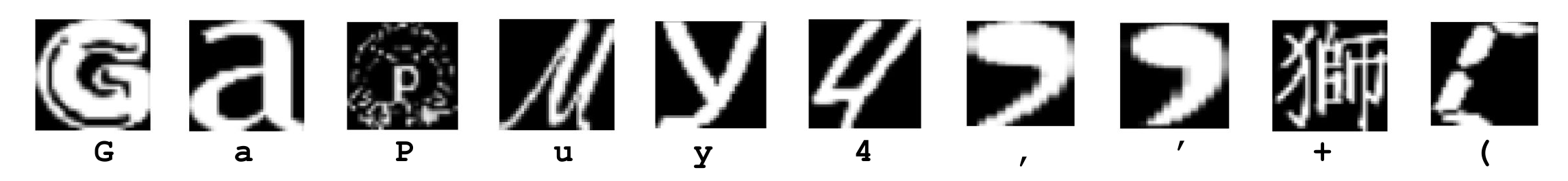}
	\caption{\textbf{Some example inputs from the Glyph dataset along with the class they are labeled as.} Note the ambiguity (between the comma and apostrophe) and the mislabeled input.}
	\label{fig:glyph-samples}
\end{figure}

To simulate the availability of private and public data 
(see \autoref{sec:pate-background}), we split data 
originally marked  as the training set (about 65M points) into partitions given to the teachers.
Each teacher is 
a ResNet~\citep{he2016deep} made of 32
leaky ReLU layers.
We train on batches of 100 inputs for 40K steps using SGD with momentum.
The learning rate, initially set to $0.1$, is decayed after 10K steps to
$0.01$ and again after 20K steps to $0.001$. These parameters
were found with a grid search.

We split holdout data in two subsets of 100K and 400K samples: the first
acts as public data to train the student and the second 
as its testing data. The student architecture
is a convolutional network 
learnt in a semi-supervised fashion with virtual adversarial training (VAT) from~\cite{miyato2017virtual}.
Using unlabeled data,
we show how VAT can regularize the student 
by making predictions constant in
\emph{adversarial}\footnote{In this context, the adversarial component refers to
the phenomenon commonly referred to as adversarial examples~\citep{biggio2013evasion,szegedy2013intriguing}
and not to the adversarial training approach taken in GANs.} directions.
Indeed, we found that GANs did not yield as much utility
for Glyph as for MNIST or SVHN.
We train with Adam
for 400 epochs and a learning rate of $6\cdot 10^{-5}$.

\subsection{Comparing the LNMax and GNMax Mechanisms}
\label{ssec:laplace-vs-gaussian}

\autoref{sec:max-of-gaussian} introduces the GNMax mechanism and
the accompanying privacy analysis. With a Gaussian distribution, whose tail diminishes more rapidly
than the Laplace distribution, we expect better utility when using the new
mechanism (albeit with a more involved privacy analysis).

To study the tradeoff between privacy and accuracy with the two mechanisms,
we run experiments training several ensembles of $M$ teachers for $M \in
\{100, 500, 1000, 5000\}$ on the Glyph data. Recall that 65 million training inputs are partitioned and
distributed among the $M$ teachers with each teacher receiving between 650K
and 13K inputs for the values of $M$ above. The test data is used to query the
teacher ensemble and the resulting labels (after the LNMax and GNMax
mechanisms) are compared with the ground truth labels provided in the dataset. This
predictive performance of the teachers is essential to good student
training with accurate labels and is a useful proxy for utility.

For each mechanism, we compute $(\eps,
\delta)$-differential privacy guarantees. As is common in literature, for
a dataset on the order of $10^8$ samples, we choose $\delta=10^{-8}$ and
denote the corresponding $\eps$ as the privacy cost. The total $\eps$ is calculated
on a subset of 4,000 queries, which is representative of the number of labels
needed by a student for accurate training (see
\autoref{ssec:student-training}). We visualize in
\autoref{fig:aggreg-lap-vs-gauss-detailed} the effect of the noise distribution (left) and the number of teachers (right) on the tradeoff between privacy costs and label accuracy.

\begin{figure}[p]
\centering
\begin{minipage}{.49\textwidth}
	\centering
	\includegraphics[width=\textwidth]{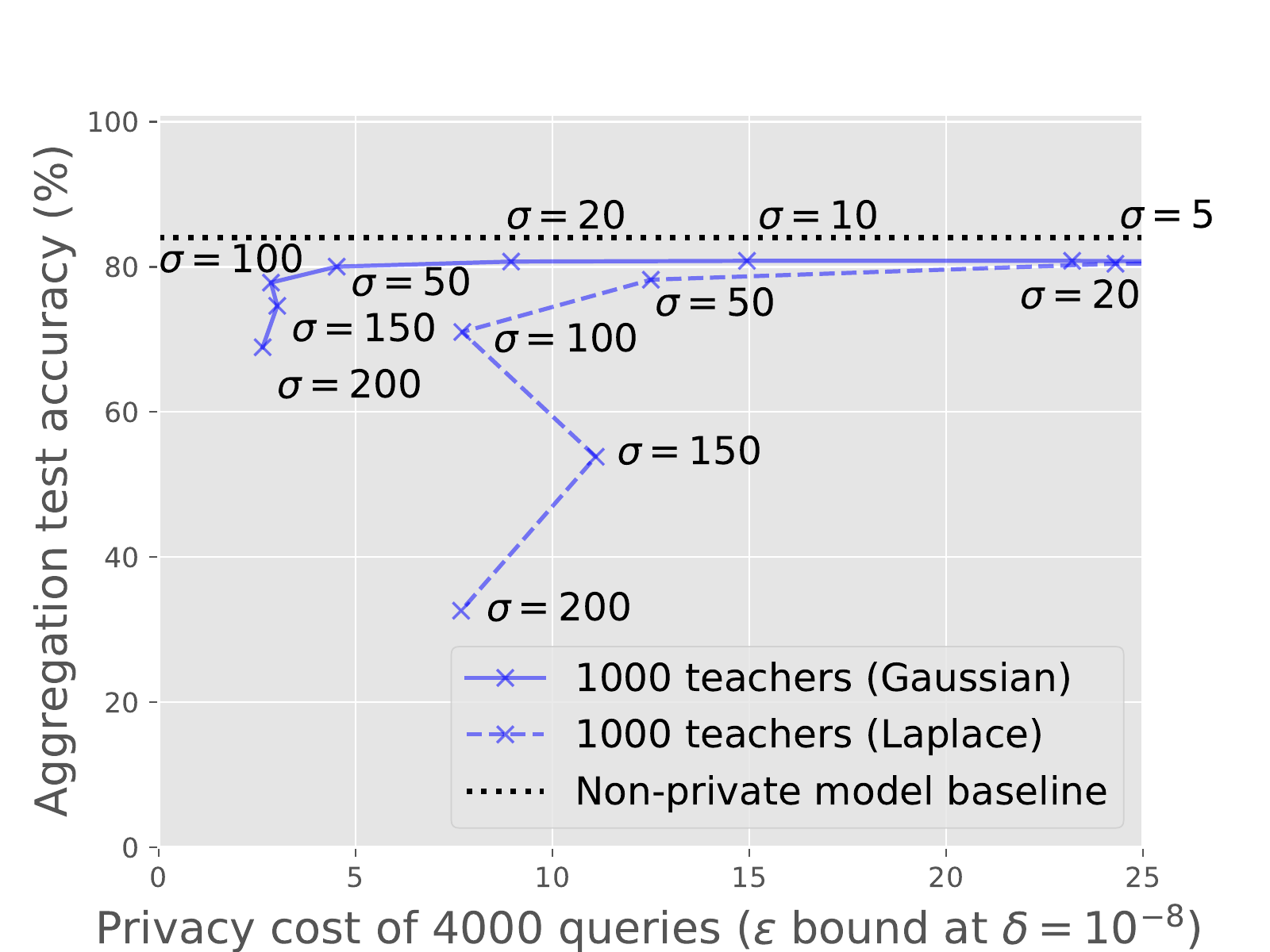}
\end{minipage}
\hfill
\begin{minipage}{.49\textwidth}
	\centering
	\includegraphics[width=\textwidth]{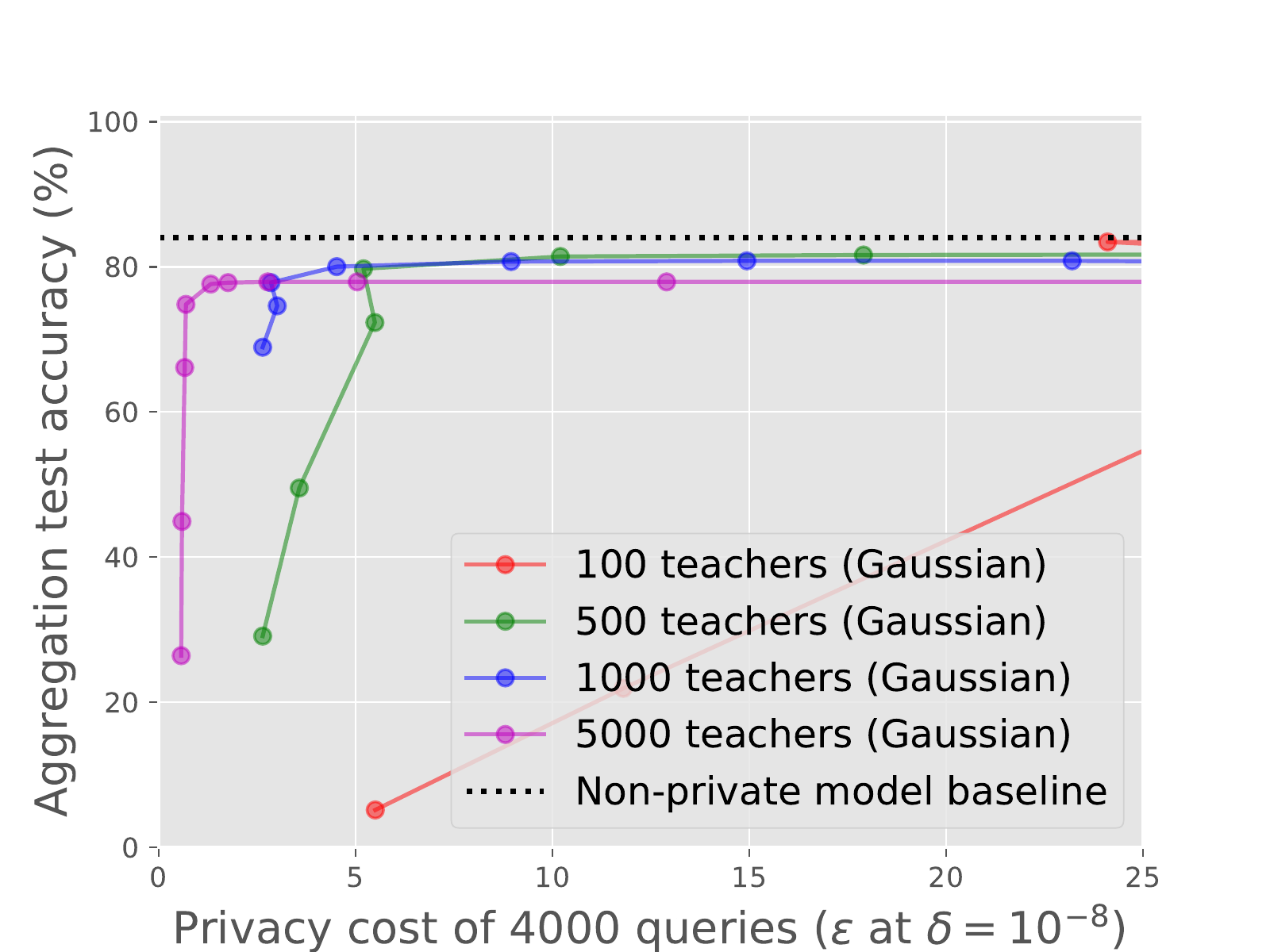}
\end{minipage}
	\caption{\textbf{Tradeoff between utility and privacy for the LNMax and GNMax aggregators on Glyph:}
	effect of the noise distribution (left) and size of the teacher ensemble (right). The LNMax aggregator uses a Laplace distribution and GNMax a Gaussian. Smaller values of the privacy cost~$\varepsilon$ (often obtained by increasing the noise scale $\sigma$---see \autoref{sec:gaussian-pate}) and higher accuracy are better.}
	\label{fig:aggreg-lap-vs-gauss-detailed}
\end{figure}

\begin{table}[p]
\centering
{\small\renewcommand{\arraystretch}{1.2}
\begin{tabular}{|c|l|c|c|c|c|}
    \hline
    &  & \textbf{Queries} & \textbf{Privacy} &  \multicolumn{2}{c|}{\textbf{Accuracy}}  \\ 
    \textbf{Dataset} & \textbf{Aggregator} & \textbf{answered} & \textbf{bound} $\boldsymbol{\varepsilon}$ &  \textbf{Student}&  \textbf{Baseline} \\ \hline \hline
    \multirow{3}{*}{MNIST} & LNMax~\citep{papernot2016semi} & 100 & 2.04 & 98.0\% & \multirow{3}{*}{99.2\%}  \\ \cline{2-5}
	 & LNMax~\citep{papernot2016semi} & 1,000 & 8.03 & 98.1\% &   \\ \cline{2-5}
	 & Confident-GNMax {\scriptsize($T\texttt{=}200,\sigma_1\texttt{=}150,\sigma_2\texttt{=}40$)} & 286 & \textbf{1.97} & \textbf{98.5\%} &  \\ \hline\hline
	 
	\multirow{3}{*}{SVHN} & LNMax~\citep{papernot2016semi} & 500 & 5.04 & 82.7\% & \multirow{3}{*}{92.8\%}  \\ \cline{2-5}
	& LNMax~\citep{papernot2016semi} & 1,000 & 8.19 & 90.7\% &   \\ \cline{2-5}
	& Confident-GNMax {\scriptsize($T\texttt{=}300,\sigma_1\texttt{=}200,\sigma_2\texttt{=}40$)} & 3,098 & \textbf{4.96} & \textbf{91.6\%} &  \\ \hline\hline
	
	\multirow{2}{*}{Adult} & LNMax~\citep{papernot2016semi} & 500 & 2.66 & 83.0\% & \multirow{2}{*}{85.0\%}  \\ \cline{2-5}
	& Confident-GNMax {\scriptsize($T\texttt{=}300,\sigma_1\texttt{=}200,\sigma_2\texttt{=}40$)} & 524 & \textbf{1.90} & \textbf{83.7\%} &  \\ \hline\hline
	
	\multirow{3}{*}{Glyph} & LNMax & 4,000 & 4.3 & 72.4\% & \multirow{3}{*}{82.2\%}  \\ \cline{2-5}
	& Confident-GNMax {\scriptsize($T\texttt{=}1000,\sigma_1\texttt{=}500,\sigma_2\texttt{=}100$)} & 10,762 & 2.03 & \textbf{75.5\%} &  \\ \cline{2-5}
	& Interactive-GNMax, two rounds 
	& 4,341 & \textbf{0.837} & 73.2\% &  \\ \hline
\end{tabular}
}
\caption{\textbf{Utility and privacy of the students.} The Confident- and
  Interactive-GNMax aggregators introduced in \autoref{sec:gaussian-pate} offer better tradeoffs between privacy (characterized by the value of the bound $\varepsilon$) and utility (the accuracy of the student compared to a non-private baseline) than the LNMax aggregator used by the original PATE proposal on all datasets we evaluated with. For MNIST, Adult, and SVHN, we use 
the labels of ensembles of $250$ teachers published by~\cite{papernot2016semi} and set $\delta=10^{-5}$ to compute values of $\varepsilon$ (to the exception of SVHN where $\delta=10^{-6}$).
All Glyph results use an ensemble of 5000 teachers
and
$\varepsilon$ is computed for $\delta=10^{-8}$. 
}
\label{table:results_summary}
\end{table}

\paragraph{Observations.}
On the left of \autoref{fig:aggreg-lap-vs-gauss}, we compare our GNMax aggregator to the LNMax aggregator
used by the original PATE proposal, on an ensemble of $1000$ teachers and
for varying noise scales $\sigma$. At fixed test accuracy, the GNMax 
algorithm
consistently outperforms the LNMax mechanism
in terms of privacy cost.
To explain this improved performance, recall notation from
\autoref{sec:max-of-gaussian}. For both
mechanisms, the data dependent privacy cost scales linearly with $\qt$---the
likelihood of an answer other than the true plurality. The value of $\qt$ falls
of as $\exp(-x^2)$ for GNMax and $\exp(-x)$ for LNMax, where $x$ is the ratio
$(n_{i^*} - n_i) / \sigma$.  Thus, when $n_{i^*} - n_i$ is (say) $4\sigma$,
LNMax would have $\qt \approx e^{-4} = 0.018...$, whereas GNMax would have
$\qt \approx e^{-16}  \approx 10^{-7}$, thereby leading to a much higher
likelihood of returning the true plurality. Moreover, this reduced $\qt$
translates to a smaller privacy cost for a given $\sigma$ leading to a better utility-privacy tradeoff.

As long as each teacher has sufficient data to learn a good-enough model,
increasing the number $M$ of teachers improves the tradeoff---as
illustrated on the right of \autoref{fig:aggreg-lap-vs-gauss-detailed} with GNMax. The larger ensembles lower the privacy
cost of answering queries by tolerating larger $\sigma$'s.
Combining the two observations made in this Figure, for a fixed label accuracy, 
we lower privacy costs by switching to the GNMax aggregator and
training a larger number $M$ of teachers.

\subsection{Student Training with the GNMax Aggregation Mechanisms}
\label{ssec:student-training}

As outlined in \autoref{sec:pate-overview}, we train a student on
public data labeled by the aggregation mechanisms. We
take advantage of PATE's flexibility and
apply the technique that performs best on each dataset: semi-supervised
learning with 
Generative Adversarial
Networks~\citep{salimans2016improved} for MNIST and SVHN,
Virtual Adversarial
Training~\citep{miyato2017virtual} for Glyph, and fully-supervised random forests
for UCI Adult. 
In addition to evaluating the total privacy
cost associated with training the student model, we compare its utility to a non-private
baseline obtained by training on the sensitive data (used to train teachers
in PATE): we use the baselines of $99.2\%$, $92.8\%$, and $85.0\%$ reported 
by~\cite{papernot2016semi} respectively for MNIST, SVHN, and UCI Adult, and we measure a baseline of
$82.2\%$ for Glyph. We compute $(\eps, \delta)$-privacy bounds and denote the
privacy cost as the $\eps$ value at a value of $\delta$ set accordingly to
number of training samples.

\paragraph{Confident-GNMax Aggregator.} Given a pool of 500 to 12,000 samples to learn from (depending on the dataset), the
student submits queries to the teacher ensemble running the Confident-GNMax
aggregator from \autoref{ssec:confident-aggregator}. A grid search over a
range of plausible values for parameters $T$, $\sigma_1$ and $\sigma_2$
yielded the values reported in \autoref{table:results_summary},
illustrating the tradeoff between utility and privacy achieved.
We additionally measure the number of queries selected by the teachers to be
answered and compare student utility to a non-private baseline.

The Confident-GNMax aggregator outperforms LNMax for the
four datasets considered in the original PATE proposal: it
reduces the privacy cost $\varepsilon$, increases  student accuracy, or 
both simultaneously.
On the uncurated Glyph data, despite the imbalance of classes and
mislabeled data (as evidenced by the 82.2\% baseline), the Confident Aggregator
achieves 73.5\% accuracy with a privacy
cost of just $\eps = 1.02$. Roughly 1,300 out of 12,000 queries
made are not answered, indicating that several expensive queries were
successfully avoided. This selectivity is analyzed in more details in \autoref{ssec:noisy-threshold-expt}.

\paragraph{Interactive-GNMax Aggregator.} On Glyph, we evaluate the utility and privacy
of an interactive training routine
that proceeds in \emph{two rounds}. Round one runs student training with a
Confident Aggregator. A grid search targeting the best privacy for
roughly 3,400 answered queries (out of 6,000)---sufficient to bootstrap a
student---led us to
setting $(T\texttt{=}3500, \sigma_1\texttt{=}1500, \sigma_2\texttt{=}100)$ and a privacy cost of $\eps \approx 0.59$.

In round two, this student was then trained with 10,000 more queries made with the Interactive-GNMax
Aggregator $(T\texttt{=}3500, \sigma_1\texttt{=}2000, \sigma_2\texttt{=}200)$.
We computed the resulting (total) privacy cost and utility at an
\emph{exemplar} data point through another grid search of plausible parameter
values. The result appears in the last row of \autoref{table:results_summary}.
With just over 10,422 answered queries in total at a
privacy cost of $\eps = 0.84$, the trained student was able to achieve 73.2\%
accuracy. 
Note that this students required
fewer answered queries compared to the Confident Aggregator.
The best overall cost of student training occurred when the privacy
costs for the first and second rounds of training were roughly the
same. (The total $\eps$ is less than $0.59 \times 2 = 1.18$ due to better
  composition---via Theorems~\ref{thm:ma_composition} and \ref{thm:ma_convert}.)

\paragraph{Comparison with Baseline.} Note that the Glyph student's accuracy remains seven
percentage
points below the non-private model's accuracy achieved by training
on the 65M sensitive inputs. We hypothesize that this is due
to the uncurated nature of the data considered. Indeed, the class imbalance
naturally requires more queries to return labels from the less represented classes.
For instance, a model trained on 200K queries is only 77\% accurate on test data.
In addition, the large fraction of mislabeled inputs are likely to
have a large privacy cost: these inputs are sensitive 
because they are outliers of the distribution, which is reflected by the weak
consensus among teachers on these inputs.

\begin{figure}[t]
\centering
\begin{minipage}{.59\textwidth}
	\centering
	\includegraphics[width=\textwidth]{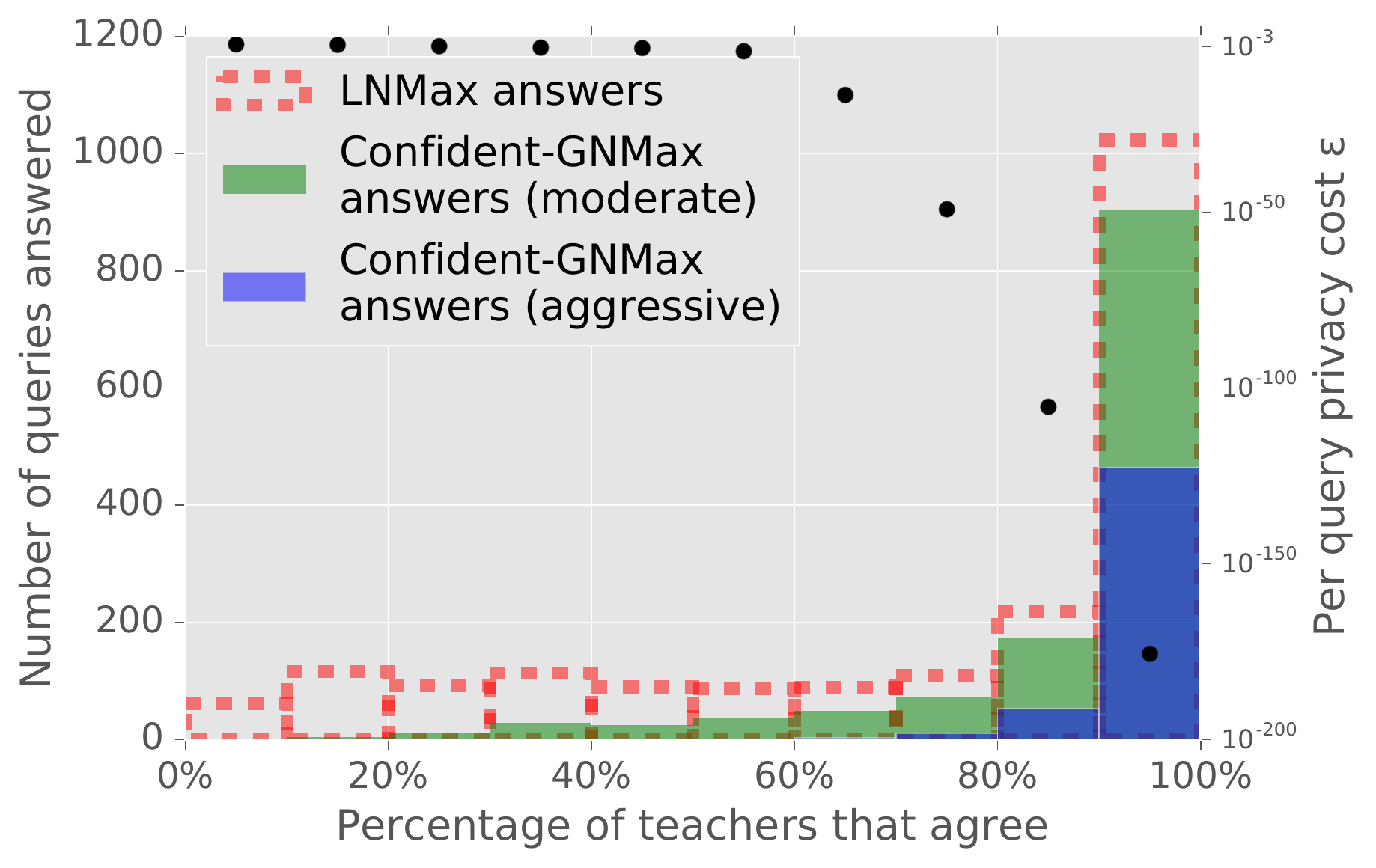}
\end{minipage}
\hfill
\begin{minipage}{.4\textwidth}
  \centering
  \includegraphics[width=\textwidth]{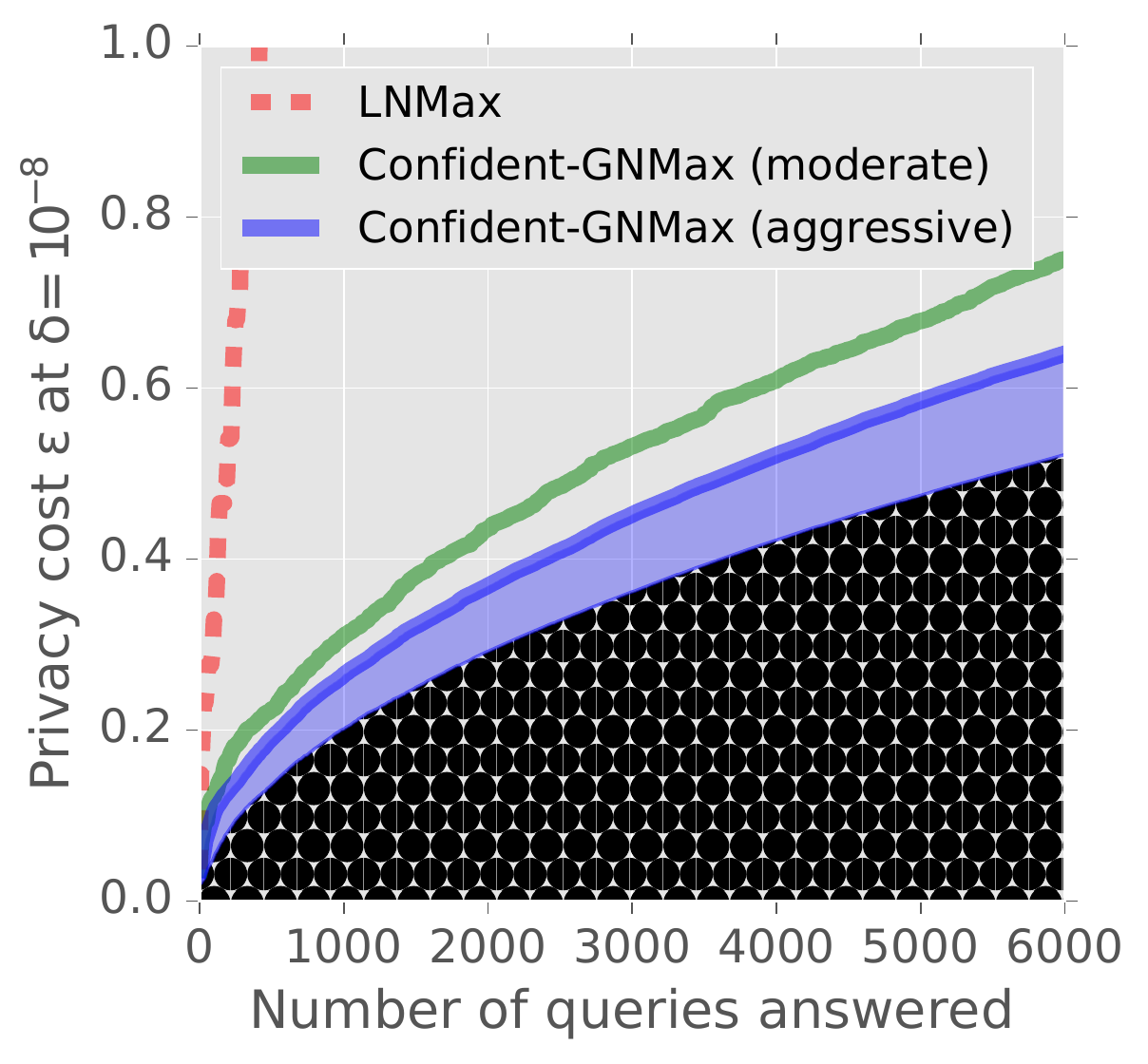}
\end{minipage}
	\caption{\textbf{Effects of the noisy threshold checking:} \emph{Left:} The number of queries 
	answered by LNMax, Confident-GNMax moderate $(T\texttt{=}3500,
	\sigma_1\texttt{=}1500)$, and Confident-GNMax aggressive $(T\texttt{=}5000,
	\sigma_1\texttt{=}1500)$. The black dots and the right axis (in log scale)
	show the expected cost of answering a single query in each bin (via GNMax,
	$\sigma_2{=}100$). \emph{Right:} Privacy cost of answering all (LNMax) vs
	only inexpensive queries (GNMax) for a given number of answered queries. The
	very dark area under the curve is the cost of selecting queries; the rest is the cost of answering them.
	}
\label{fig:threshold-check}
\end{figure}

\subsection{Noisy Threshold Checks and Privacy Costs}
\label{ssec:noisy-threshold-expt}

Sections~\ref{sec:max-of-gaussian} and \ref{ssec:confident-aggregator} motivated the need for a noisy
threshold checking step before having the teachers answer queries: it
prevents most of the privacy budget being consumed by few queries that are
expensive and also likely to be incorrectly answered.
In \autoref{fig:threshold-check}, we compare the privacy cost $\varepsilon$ of answering
all queries to only answering confident queries for a fixed number of queries.

We run additional experiments to support the evaluation
from \autoref{ssec:student-training}. 
With the votes of 5{,}000 teachers on the Glyph dataset, we plot in \autoref{fig:threshold-check} the histogram of the plurality vote
counts ($n_{i^*}$ in the notation of \autoref{sec:max-of-gaussian}) across
25{,}000 student queries. 
We
compare these values to the vote counts of queries that passed the noisy threshold
check for two sets of parameters $T$ and $\sigma_1$ in \autoref{alg:confident aggregator}. 
Smaller values
imply weaker teacher agreements and consequently more expensive queries.

When $(T\texttt{=}3500, \sigma_1\texttt{=}1500)$ we capture
a significant fraction of queries where teachers have a strong
consensus (roughly $> 4000$ votes) while managing to filter out many queries with poor consensus. This
\emph{moderate check} ensures that although many queries with plurality votes between
2,500 and 3,500 are answered (i.e., only 50--70\% of teachers agree on a
label) the expensive ones are most likely discarded. 
For $(T\texttt{=}5000, \sigma_1\texttt{=}1500)$, queries with poor consensus are completely
culled out.
This selectivity comes at the expense of a noticeable
drop for queries that might have had a strong
consensus and little-to-no privacy cost. 
Thus, this \emph{aggressive check} answer fewer queries 
with very strong privacy
guarantees. We
reiterate that this threshold checking step
itself is done in a private manner. Empirically, in our Interactive
Aggregator experiments, we expend about a third to
a half of our privacy budget on this step, which still yields a very
small cost \emph{per query} across 6,000 queries.

\makeatletter{}\section{Conclusions}\label{sec:conclude}

The key insight motivating the addition of a noisy thresholding step
to the two aggregation mechanisms proposed in our work is that there is a form of synergy
between the privacy and accuracy of labels output by the aggregation:
\emph{labels that come at a small privacy cost also happen to be more
likely to be correct}. As a consequence, we are able to provide more quality
supervision to the student by choosing not to output labels when the
consensus among teachers is too low to provide an aggregated prediction
 at a small cost in privacy. This observation was further confirmed in some
 of our experiments where we observed that if we trained the student
 on either private or non-private labels, the former
 almost always gave better performance than the latter---for a fixed number of labels.

Complementary with these aggregation mechanisms is the use of a Gaussian
 (rather than Laplace) distribution to perturb teacher votes. In our experiments
 with Glyph data, these changes proved essential to preserve the accuracy of the
 aggregated labels---because of the large number of classes. The analysis
 presented in \autoref{sec:gaussian-pate} details the delicate but
 necessary adaptation of analogous results for the Laplace NoisyMax.

 As was the case for the original PATE proposal, semi-supervised learning
 was instrumental to ensure the student achieves strong utility
 given a limited set of labels from the aggregation mechanism. However,
 we found that virtual adversarial training outperforms the approach from~\cite{salimans2016improved}
 in our experiments with Glyph data. 
These results establish lower bounds on the performance that a student
can achieve when supervised with our aggregation mechanisms; 
future work may continue to investigate
virtual adversarial training, semi-supervised generative adversarial networks
and other techniques for learning the student in these particular settings
 with restricted supervision.

\makeatletter{}\section*{Acknowledgments}
We are grateful to Mart\'in Abadi, Vincent Vanhoucke, and Daniel Levy for
their useful inputs and discussions towards this paper.

\bibliographystyle{iclr2018_conference}

\newpage

\appendix
\makeatletter{}\section{Appendix: Privacy Analysis}
\label{ap:privacy-analysis}

In this appendix, we provide the proofs of \autoref{theorem:higher-to-lower} and \autoref{prop:gaussian-q}.
Moreover, we present \autoref{prop:gaussian-optimal}, which provides optimal
values of $\hi_1$ and $\hi_2$ to apply towards
\autoref{theorem:higher-to-lower} for the GNMax mechanism. We start off
with a statement about the \renyi differential privacy guarantee of the
GNMax.

\begin{proposition}\label{prop:gaussian_rdp}
The GNMax aggregator $\GM{\sigma}$ guarantees $\left(\order, \order/\sigma^2\right)$-RDP for all $\order\geq 1$.
\end{proposition}
\begin{proof}
  \vspace{-0.1in}
The result follows from observing that $\GM{\sigma}$ can be
  decomposed into applying the $\argmax$ operator to a noisy histogram
  resulted from adding Gaussian noise to each dimension of the original histogram. The Gaussian mechanism
  satisfies $(\order, \order/2\sigma^2)$-RDP \citep{mironov2017renyi}, and since each teacher may change
  two counts (incrementing one and decrementing the other), the overall RDP
  guarantee is as claimed.
\end{proof}

\edef\oldtheoremcounter{\the\value{theorem}}
\setcounter{theorem}{\lemmaqt}
\begin{proposition}
\newcounter{propositionrestated}
\setcounter{propositionrestated}{\lemmaqt}
\refstepcounter{propositionrestated}\label{prop:gaussian-q-full}
For a GNMax aggregator $\GM{\sigma}$, the teachers' votes histogram  $\barn=(n_1,\dots,n_m)$, and for any $i^* \in \kset{m}$, we have 
\[
\Pr\left[\GM{\sigma}(D) \neq i^*\right] \leq q(\barn),
\]
where
\[
q(\barn)\eqdef \frac{1}{2} \sum_{i\neq i^*} \mathrm{erfc}\left(\frac{n_{i^*}-n_i}{2\sigma}\right).
\]
\end{proposition}
\setcounter{theorem}{\oldtheoremcounter}
\begin{proof}
Recall that $\GM\sigma{(D)}=\argmax(n_i+Z_i)$, where $Z_i$ are distributed as $\calN(0,\sigma^2)$. Then for any $i^* \in \kset{m}$, we have
\begin{align*}
  \Pr[\GM{\sigma}(D) \neq i^*] = \Pr\left[\exists i, n_i+Z_i > n_{i^*}+Z_{i^*}\right]
  &\leq \sum_{i\neq i^*} \Pr\left[n_i+Z_i > n_{i^*}+Z_{i^*}\right] \\
  &=   \sum_{i\neq i^*} \Pr\left[Z_i - Z_{i^*} > n_{i^*}-n_i\right] \\
  &=   \sum_{i\neq i^*} \frac{1}{2}\left(1-\erf\left(\frac{n_{i^*}-n_i}{2\sigma}\right)\right).
\end{align*}
where the last equality follows from the fact that $Z_i - Z_j$ is a Gaussian random variable with mean zero and variance $2\sigma^2$.
\end{proof}

We now present a precise statement of~\autoref{theorem:higher-to-lower}.

\edef\oldtheoremcounter{\the\value{theorem}}
\setcounter{theorem}{\thmtransfer}
\begin{theorem}
\newcounter{theoremrestated}
\setcounter{theoremrestated}{\thmtransfer}
\refstepcounter{theoremrestated}\label{theorem:higher-to-lower-full}
Let $\calM$ be a randomized algorithm with $\left(\hi_1,
  \eps_1\right)$-RDP and $\left(\hi_2, \eps_2\right)$-RDP guarantees and suppose that there exists \emph{a likely} outcome $i^*$ given a dataset $D$ and a bound $\qt\leq 1$ such that $\qt
  \geq \Pr\left[\calM(D) \neq i^* \right]$. Additionally suppose that $\low \leq \hi_1$ and $\qt \leq e^{(\hi_2-1) \eps_2} / \power{\frac{\hi_1}{\hi_1-1} \cdot \frac{\hi_2}{\hi_2-1}}{\hi_2}$.
Then, for any neighboring dataset $D'$ of $D$, we have:
\begin{align}\label{eqn:renyi-data-dependent}
  \Div{\low}{\calM(D)}{\calM(D')} \leq
  \frac{1}{\low - 1} \log \Big(
  \left(1-\qt\right) \cdot \boldsymbol{A}(\qt, \hi_2, \eps_2)^{\low-1} + \qt
  \cdot \boldsymbol{B}(\qt, \hi_1, \eps_1)^{\low-1}
  \Big)
\end{align}
  where $\boldsymbol{A}(\qt, \hi_2, \eps_2) \eqdef (1-\qt) / \left(1-\left(\qt
  e^{\eps_2} \right)^{\frac{\hi_2-1}{\hi_2}} \right)$ and $\boldsymbol{B}(\qt,
  \mu_1, \eps_1) \eqdef e^{\eps_1} / \qt^{\frac1{\hi_1-1}}$.
\end{theorem}
\setcounter{theorem}{\oldtheoremcounter}

\begin{proof}
Before we proceed to the proof, we introduce some simplifying notation. For a
randomized mechanism $\calM$ and neighboring datasets $D$ and $D'$, we define
  \begin{align*}
    \MAB_{\calM}(\order; D, D') &\eqdef \Div{\order}{\calM(D)}{\calM(D')}\\
    &= \frac{1}{\order - 1} \log \expt{x\sim \calM(D)}{\power{\frac{\Pr\left[\calM(D)=x
  \right]\hfill}{\Pr\left[\calM(D')=x\right]}}{\order-1}}.
  \end{align*}

As the proof involves working with the RDP bounds in the exponent, we
set $\zeta_1 \eqdef e^{\eps_1 (\hi_1-1)}$ and $\zeta_2 \eqdef e^{\eps_2 (\hi_2-1)}$.

Finally, we define the following shortcuts:
\begin{align*}
q_i&\eqdef \Pr\left[\calM(D)=i\right]\textrm{ and }q\eqdef \sum_{i\neq i^*} q_i=\Pr\left[\calM(D)\neq i^*\right],\\
p_i&\eqdef \Pr\left[\calM(D')=i\right]\textrm{ and }p\eqdef \sum_{i\neq i^*} p_i=\Pr\left[\calM(D')\neq i^*\right],
\end{align*}
and note that $q \leq \qt$. 

From the
	definition of \renyi differential privacy, $(\mu_1, \eps_1)$-RDP implies:
  \begin{align}
    \exp\left(\MAB_\calM(\hi_1; D, D')\right)
    =\left(
    \frac{(1-q)^{\hi_1}}{(1-p)^{{\hi_1-1}}} + \sum_{i\neq i*} \frac{q_i^{\hi_1}}{p_i^{{\hi_1-1}}}
    \right)^{1/(\hi_1-1)}
    &\leq \exp(\eps_1)\nonumber \\
    \implies
    \sum_{i>1} \frac{q_i^{\hi_1}}{p_i^{{\hi_1-1}}} = \sum_{i>1} q_i \power{\frac{q_i}{p_i}}{{\hi_1}-1} &\leq {\zeta_1}.\label{eq:bounding-i->-1}
\end{align}
Since ${\hi_1} \geq \low$, $f(x)\eqdef x^{\frac{\hi_1-1}{\low-1}}$ is convex.
Applying Jensen's Inequality we have the following:
\begin{align}\label{pf:high to low part 1}
\power{\frac{\sum_{i\neq i^*} q_i
  \power{\frac{q_i}{p_i}}{\low-1}}{q}}{\frac{\hi_1-1}{\low-1}} &\leq \frac{\sum_{i\neq i^*} q_i \power{\frac{q_i}{p_i}}{{\hi_1-1}}}{q} \nonumber  \\
\implies \sum_{i\neq i^*} q_i \power{\frac{q_i}{p_i}}{\low-1} &\leq q \power{\frac{\sum_{i\neq i^*} q_i \power{\frac{q_i}{p_i}}{{\hi_1-1}}}{q}}{\frac{\low-1}{\hi_1-1}} \nonumber  \\
\overset{(\ref{eq:bounding-i->-1})}{\implies}\sum_{i\neq i^*} q_i \power{\frac{q_i}{p_i}}{\low-1} &\leq {\zeta_1}^{\frac{\low-1}{\hi_1-1}} \cdot
  q^{1-\frac{\low-1}{\hi_1-1}}.
\end{align}

Next, by the bound at order $\hi_2$, we have:
\begin{align*}
  \exp\left(\MAB_\calM({\hi_2}; D', D)\right) = \left(
  \frac{(1-p)^{\hi_2}}{(1-q)^{{\hi_2}-1}} + \sum_{i\neq i^*} \frac{p_i^{\hi_2}}{q_i^{{\hi_2}-1}}
  \right)^{1/(\hi_2-1)}&\leq \exp(\eps_2)\\
  \implies 
    \frac{(1-p)^{\hi_2}}{(1-q)^{{\hi_2}-1}} + \sum_{i\neq i^*} \frac{p_i^{\hi_2}}{q_i^{{\hi_2}-1}}
  &\leq \zeta_2.
\end{align*}
By the data processing inequality of \renyi divergence, we have
\begin{align*}
\frac{(1-p)^{\hi_2}}{(1-q)^{{\hi_2}-1}} + \frac{p^{\hi_2}}{q^{{\hi_2}-1}}
\leq {\zeta_2},
\end{align*}
which implies $\frac{p^{\hi_2}}{q^{{\hi_2}-1}} \leq {\zeta_2}$ and thus
\begin{align}\label{pf:high to low part 2}
p \leq \left(q^{\hi_2-1}\zeta_2\right)^{\frac{1}{\hi_2}}.
\end{align}
\\
Combining \eqref{pf:high to low part 1} and \eqref{pf:high to low part 2}, we can derive a bound at  $\low$.
\begin{align}
\exp\left(\MAB_\calM(\low, D, D')\right)
  &= \left(
  \frac{(1-q)^{\low}}{(1-p)^{\low-1}} + \sum_{i\neq i^*} \frac{q_i^{\low}}{p_i^{\low-1}}
  \right)^{1/(\low-1)} \nonumber \\
  &\leq \left(
\frac{(1-q)^{\low}}{\left(1-(q^{{\hi_2-1}}
  {\zeta_2})^{\frac{1}{\hi_2}}\right)^{\low-1}}
   + {\zeta_1}^{\frac{\low-1}{\hi_1-1}} \cdot
     q^{1-\frac{\low-1}{\hi_1-1}} \right)^{1/(\low-1)}.
   \label{eq:moment_q}
\end{align}

Although Equation~\eqref{eq:moment_q} is very close to the corresponding statement in the theorem's claim, one subtlety remains. The bound~\eqref{eq:moment_q} applies to the exact probability $q = \Pr\left[\calM(D) \neq i^*
\right]$. In the theorem statement, and in practice, we can only derive an
upper bound $\qt$ on $\Pr\left[\calM(D) \neq i^* \right]$. The last step of
the proof requires showing that the expression in Equation~\eqref{eq:moment_q} is 
monotone in the range of values of $q$ that we care about. 

\begin{lemma}[Monotonicity of the bound]
Let the functions $f_1(\cdot)$ and $f_2(\cdot)$ be
\begin{align*}
f_1(x) \eqdef \frac{(1-x)^{\low}}{\left(1-(x^{{\hi_2-1}}
  {\zeta_2})^{\frac{1}{\hi_2}}\right)^{\low-1}} \qquad \text{and} \qquad
f_2(x) &\eqdef {\zeta_1}^{\frac{\low-1}{\hi_1-1}} \cdot  x^{1-\frac{\low-1}{\hi_1-1}},
\end{align*}
Then $f_1(x) + f_2(x)$ is increasing in
$\left[0, \min\left(1, \zeta_2 /
\power{\frac{\hi_1}{{\hi_1-1}} \cdot \frac{\hi_2}{{\hi_2-1}}}{\hi_2}\right)
\right]$.
\end{lemma}

\begin{proof}
Taking the derivative of $f_1(x)$, we have:
\begin{align*}
f'_1(x) &=
 \frac
  {-\low(1-x)^{\low-1} (1-(x^{{\hi_2-1}} {\zeta_2})^{\frac{1}{\hi_2}})^{\low-1}}
{(1-(x^{{\hi_2-1}} {\zeta_2})^{\frac1{\hi_2}})^{2\low-2}}\\
  &\qquad \qquad  + \frac{(1-x)^{\low} (\low-1) (1-(x^{{\hi_2-1}} {\zeta_2})^{\frac1{\hi_2}})^{\low-2}{\zeta_2}^{\frac1{\hi_2}} \cdot
  \frac{\hi_2-1}{\hi_2} \cdot x^{-\frac1{\hi_2}}}
{(1-(x^{{\hi_2-1}} {\zeta_2})^{\frac1{\hi_2}})^{2\low-2}}\\
& =
\frac{(1-x)^{\low-1}}{{(1-(x^{{\hi_2-1}} {\zeta_2})^{\frac1{\hi_2}})^{\low-1}}}
\left(
-\low
+
(\low-1)\left(1-\frac1{\hi_2}\right)
\frac{1-x}{1-(x^{{\hi_2-1}} {\zeta_2})^{\frac1{\hi_2}}}
\power{\frac{{\zeta_2}}{x}}{\frac{1}{\hi_2}}
\right).
\end{align*}

We intend to show that: \begin{align} f'_1(x) \geq
-\low
+
(\low-1)\left(1-\frac1{\hi_2}\right)
\power{\frac{{\zeta_2}}{x}}{\frac{1}{\hi_2}}
. \label{eq:f1-lb}
\end{align}

For $x \in \left[0, \zeta_2 /
\power{\frac{\hi_1}{{\hi_1-1}} \cdot
\frac{\hi_2}{\hi_2-1}}{\hi_2}\right]$
and $y \in [1, \infty)$, define $g(x, y)$ as:
\begin{align*}
  g(x, y)&\eqdef
-\low\cdot y^{\low-1}
+
(\low-1)\left(1-\frac1{\hi_2}\right)
\power{\frac{{\zeta_2}}{x}}{\frac{1}{\hi_2}}y^\low.
\end{align*}

We claim that $g(x, y)$ is increasing in $y$ and therefore $g(x, y) \geq g(x,
1)$, and prove it by showing the partial derivative of $g(x, y)$ with respect to $y$ is non-negative. Take a derivative with respect to $y$ as:
\begin{align*}
g'_y(x, y) &=
-\low(\low-1)y^{\low-2}  +
\low(\low-1)\left(1-\frac1{\hi_2}\right)
\power{\frac{{\zeta_2}}{x}}{\frac{1}{\hi_2}}y^{\low-1}\\
&= \low (\low-1)y^{\low-2}
\left(
-1 +
\left(1-\frac1{\hi_2}\right)
\power{\frac{{\zeta_2}}{x}}{\frac{1}{\hi_2}}y
\right).
\end{align*}

To see why $g'_y(x, y)$ is non-negative in the respective ranges of $x$ and
$y$, note that:
\begin{align*}
  x\leq \zeta_2 / \power{\frac{\hi_1}{\hi_1-1} \cdot
  \frac{\hi_2}{\hi_2-1}}{\hi_2} &\implies x \leq \zeta_2 /
  \power{\frac{\hi_2}{{\hi_2-1}}}{\hi_2}\\ 
  &\implies 1 \leq
  \frac{\zeta_2}{x} \cdot \left( \frac{\hi_2-1}{\hi_2}\right)^{\hi_2}\\
  &\implies 1 \leq \frac{\hi_2-1}{\hi_2}
\power{\frac{{\zeta_2}}{x}}{\frac{1}{\hi_2}}\\ 
&\implies
  1 \leq \frac{{\hi_2-1}}{\hi_2}
\power{\frac{{\zeta_2}}{x}}{\frac{1}{\hi_2}} y\tag{as
	$y \geq 1$}\\ &\implies
  0 \leq -1 + \frac{{\hi_2-1}}{\hi_2}
\power{\frac{{\zeta_2}}{x}}{\frac{1}{\hi_2}} y\\
  &\implies 0 \leq g'_y(x, y).\tag{in the resp.~range of $x$ and
  $y$}
\end{align*}

Consider $\frac{1-x}{1-(x^{{\hi_2-1}} {\zeta_2})^{1/{\hi_2}}}$. Since $\zeta_2 \geq 1$ 
and $x \leq 1$, we have $x \leq \zeta_2$ and hence 
\[\frac{1-x}{1-(x^{{\hi_2-1}} {\zeta_2})^{\frac{1}{\hi_2}}} \geq \frac{1-x}{1-(x^{{\hi_2-1}} {x})^{\frac{1}{{\hi_2}}}} = 1.\]

Therefore we can set $y = \frac{1-x}{1-(x^{{\hi_2-1}} {\zeta_2})^{1/{\hi_2}}}$ and apply the fact that $g(x, y) \geq g(x,1)$ for all $y\geq 1$ to get
\begin{align*}
f'_1(x) \geq
-\low
+
(\low-1)\left(1-\frac1{\hi_2}\right)
\power{\frac{{\zeta_2}}{x}}{\frac{1}{\hi_2}},
\end{align*}
as required by \eqref{eq:f1-lb}.

Taking the derivative of $f_2(x)$, we have:
\begin{align*}
f'_2(x) =
{\zeta_1}^{\frac{\low-1}{\hi_1-1}} \cdot  \left(1-\frac{\low-1}{\hi_1-1}\right)x^{-\frac{\low-1}{\hi_1-1}}
= \power{\frac{{\zeta_1}}{x}}{\frac{\low-1}{{\hi_1-1}}} \left(1-\frac{\low-1}{{\hi_1-1}}
  \right)
\geq 1-\frac{\low-1}{{\hi_1-1}}.
\end{align*}
Combining the two terms together, we have:
\begin{align*}
f'(x) &\geq
-\low
+
(\low-1)\left(1-\frac1{\hi_2}\right)
\power{\frac{{\zeta_2}}{x}}{\frac{1}{\hi_2}}
+ 1-\frac{\low-1}{{\hi_1-1}}\\
&=
(\low-1) \left(-\frac{\hi_1}{{\hi_1-1}}
+
\frac{{\hi_2-1}}{\hi_2}
\power{\frac{{\zeta_2}}{x}}{\frac{1}{\hi_2}}
\right).
\end{align*}
For $f'(x)$ to be non-negative we need:
\begin{align*}
-\frac{\hi_1}{{\hi_1-1}}
+
\frac{{\hi_2-1}}{{\hi_2}}
\power{\frac{{\zeta_2}}{x}}{\frac{1}{\hi_2}} &\geq 0 \\
\iff 
\power{\frac{\hi_1}{\hi_1-1} \cdot \frac{\hi_2}{\hi_2-1}}{\hi_2} &\leq \frac{{\zeta_2}}{x}.
\end{align*}
So $f(x)$ is increasing for $x \in \left[0, {\zeta_2} / \power{\frac{\hi_1}{\hi_1-1} \cdot \frac{\hi_2}{\hi_2-1}}{\hi_2}\right]$.
This means for $q \leq \qt \leq {\zeta_2} / \power{\frac{\hi_1}{\hi_1-1} \cdot \frac{\hi_2}{\hi_2-1}}{\hi_2}$, we have $f(q) \leq f(\qt)$. This
completes the proof of the lemma and that of the theorem.
\end{proof}
\end{proof}

\autoref{theorem:higher-to-lower-full} yields data-dependent \renyi
differential privacy bounds for \emph{any} value of $\hi_1$ and $\hi_2$
larger than $\order$. The following proposition simplifies this search by
calculating optimal higher moments $\hi_1$ and $\hi_2$ for the GNMax mechanism
with variance $\sigma^2$.

\begin{proposition}\label{prop:gaussian-optimal}
When applying  \autoref{theorem:higher-to-lower-full} and \autoref{prop:gaussian_rdp} for GNMax with Gaussian of variance $\sigma^2$, the right-hand side of \eqref{eqn:renyi-data-dependent} is minimized at
\begin{align*}
  {\hi_2} &= \sigma \cdot \sqrt{\log (1/\qt)}, \text{  and 
  } {\hi_1} = \hi_2 + 1.
\end{align*}
\end{proposition}
\begin{proof}
  We can minimize both terms in \eqref{eqn:renyi-data-dependent} independently. To minimize the first term in
  \eqref{theorem:higher-to-lower}, we minimize $\left(\qt
    e^{\eps_2} \right)^{1-1/\hi_2}$ by considering logarithms:
\begin{align*}
  \log \left\{\left(\qt
      e^{\eps_2} \right)^{1-1/\hi_2} \right\}
  &= \log \left\{ \qt^{1 - \frac{1}{\hi_2}} \exp\left(\frac{{\hi_2-1}}{\sigma^2}\right)
  \right\}\\
  &= \left(1 - \frac{1}{\hi_2}\right) \cdot \log \qt +
  \frac{{\hi_2}-1}{\sigma^2}\\
  &= \frac{1}{\hi_2} \log \frac{1}{\qt} + \frac{\hi_2}{\sigma^2} -
  \frac{1}{\sigma^2}  - \log \frac{1}{\qt},
\end{align*}
  which is minimized at ${\hi_2} = \sigma \cdot \sqrt{\log (1/\qt)}.$
  \medskip

To minimize the second term in \eqref{theorem:higher-to-lower}, we minimize
  $e^{\eps_1}/\qt^{1/{(\hi_1-1)}}$ as follows:
\begin{align*}
  \log \left\{ \frac{e^{\eps_1}}{\qt^{1/{(\hi_1-1)}}} \right\}
  &= \log \left\{\qt^{-1/{(\mu_1-1)}}\exp\left(\frac{\hi_1}{\sigma^2}\right)
  \right\}\\
  &= \frac{{\hi_1}}{\sigma^2} + \frac{1}{{\hi_1}-1} \log \frac{1}{\qt} \\
  &= \frac{1}{\sigma^2} + \frac{{\hi_1-1}}{\sigma^2} + \frac{1}{{\hi_1-1}} \log
  \frac{1}{\qt},
\end{align*}
  which is minimized at ${\hi_1} = 1+\sigma \cdot \sqrt{\log (1/\qt)}$ completing the proof.
\end{proof}

Putting this together, we apply the following steps to calculate RDP of order \order\ for GNMax with variance $\sigma^2$ on a given dataset $D$. First, we compute a bound $q$ according to \autoref{prop:gaussian-q-full}. Then we use the smaller of two bounds: a data-dependent (\autoref{theorem:higher-to-lower-full}) and a
data-independent one (\autoref{prop:gaussian_rdp}) :
\begin{equation*}
\MA_\sigma(q)\eqdef \min \left\{\frac1{\order-1}\log \left\{
\left(1-q\right) \cdot \boldsymbol{A}(q, \hi_2, \eps_2)^{\low-1} + q
\cdot \boldsymbol{B}(q, \hi_1, \eps_1)^{\low-1}\right\}, \order/\sigma^2
\right\},
\end{equation*}
where $\boldsymbol{A}$ and $\boldsymbol{B}$ are defined as in the statement of
\autoref{theorem:higher-to-lower-full}, the parameters $\mu_1$ and $\mu_2$ are
selected according to \autoref{prop:gaussian-optimal}, and
$\eps_1\eqdef\mu_1/\sigma^2$ and $\eps_2\eqdef\mu_2/\sigma^2$
(\autoref{prop:gaussian_rdp}). Importantly, the first expression is evaluated
only when $q<1$, $\mu_1\geq \lambda$, $\mu_2> 1$, and $q \leq e^{(\hi_2-1)
\eps_2} / \power{\frac{\hi_1}{\hi_1-1} \cdot \frac{\hi_2}{\hi_2-1}}{\hi_2}$.
These conditions can either be checked for each application of the aggregation mechanism, or a critical value of $\qswitch$ that separates the range of applicability of the data-dependent and data-independent bounds can be computed for given $\sigma$ and \order. In our implementation we pursue the second approach.

The following corollary offers a simple asymptotic expression of the privacy of GNMax for the case when there are large (relative to $\sigma$) gaps between the highest \emph{three} vote counts.

\begin{corollary}
If the top three vote counts are $n_1>n_2>n_3$ and $n_1-n_2$, $n_2-n_3\gg \sigma$, then  the mechanism GNMax with Gaussian of variance $\sigma^2$ satisfies $(\order, \exp(-2\order/\sigma^2)/\order)$-RDP for $\order=(n_1-n_2)/4$.
\end{corollary}

\begin{proof}
Denote the noisy counts as $\tilde{n}_i=n_i+\calN(0, \sigma^2)$. Ignoring outputs other than those with the highest and the second highest counts, we bound $q = \Pr\left[\calM(D) \neq 1\right]$ as $\Pr[\tilde{n}_1 < \tilde{n}_2]=\Pr[N(0, 2\sigma^2)>n_1-n_2]< \exp\left(-(n_1-n_2)^2/4\sigma^2\right)$, which we use as $\qt$. Plugging $\qt$ in \autoref{prop:gaussian-optimal}, we have $\mu_1-1=\mu_2=(n_1-n_2)/2$, limiting the range of applicability of \autoref{theorem:higher-to-lower-full} to $\order<(n_1-n_2)/2$.

Choosing $\order=(n_1-n_2)/4$ ensures $\boldsymbol{A}(\qt, \mu_2,\eps_2)\approx 1$, which allows approximating the bound~(\ref{eqn:renyi-data-dependent}) as $\qt\cdot \boldsymbol{B}(\qt, \mu_1,\eps_1)^{\order-1}/(\order-1)$. The proof follows by straightforward calculation.
\end{proof}

\makeatletter{}\section{Smooth Sensitivity and Publishing the Privacy Parameter}\label{ap:smooth-sensitivity}

The privacy guarantees obtained for the mechanisms in this paper via~\autoref{theorem:higher-to-lower-full} take as input $\qt$, an upper bound on the probability
that the aggregate mechanism returns the true plurality. This means that the
resulting privacy parameters computed depend on teacher votes and hence the
underlying data. To avoid potential privacy breaches from simply publishing
the data-dependent parameter, we need to publish a \emph{sanitized version} of
the privacy loss. This is done by adding noise to the computed privacy loss
estimates using the \emph{smooth sensitivity} algorithm proposed by
\citet{nissim2007smooth}. 

This section has the following structure. First we recall the notion of smooth sensitivity
and introduce an algorithm for computing the smooth sensitivity of the privacy loss
function of the GNMax mechanism. In the rest of the section we prove correctness of these 
algorithms by stating several conditions on the mechanism, proving that these conditions
are sufficient for correctness of the algorithm, and finally demonstrating that GNMax satisfies
these conditions.

\subsection{Computing Smooth Sensitivity}
\label{sec:computing-smooth-sensitivity}

Any dataset $D$ defines a histogram $\barn=(n_1,\dots,n_m) \in \mathbb{N}^m$ of the teachers' votes. We have a natural notion of the distance between two histograms $\dist(\barn, \barn')$ and a function $q\colon \mathbb{N}^m \to [0, 1]$ on these
histograms computing the bound according to \autoref{prop:gaussian-q-full}.  The value $q(\barn)$ can be used as $\qt$ in the application of \autoref{theorem:higher-to-lower-full}. Additionally we have $n^{(i)}$ denote the $i$-th highest bar in the histogram.

We aim at calculating a smooth sensitivity of $\ma{q(\barn)}$ whose
definition we recall now.

\begin{definition}[Smooth Sensitivity]\label{def:smooth-sensitivity}
	Given the smoothness parameter $\beta$, a $\beta$-\emph{smooth sensitivity} of $f(n)$ is defined as 
	\begin{gather*}
	\SS_\beta(\barn) \eqdef
	\max_{d\geq 0} e^{-\beta d} \cdot \max_{\barn'\colon \dist(\barn, \barn') \leq  d} \LS(\barn'),\\
	\intertext{where}
	\LS(\barn) \geq \max_{\barn'\colon \dist(\barn, \barn') = 1} |f(n) - f(n')|
	\end{gather*}
	is an upper bound on the local sensitivity. 
\end{definition}

We now describe Algorithms~\ref{alg:ls}--\ref{alg:ss} computing a smooth sensitivity of $\ma{q(\cdot)}$. The algorithms assume the existence of  efficiently computable functions $q\colon \mathbb{N}^m\to [0,1]$, $\ql,\qu\colon [0,1]\to[0,1]$, and a constant $\qswitch$.

Informally, the functions $\qu$ and $\ql$ respectively upper and lower bound the value of $q$ evaluated at any neighbor of~$\barn$ given $q(\barn)$, and $[0, \qswitch)$ limits the range of applicability of data-dependent analysis. 

The functions $\ql$ and $\qu$ are defined as follows. Their derivation appears in \autoref{ss:gaussian-satisfies-conditions}.
\begin{align*}
\qu(q) &\eqdef \min\left\{\frac{m-1}{2} \erfc\left(\erfcinv\left(\frac{2q}{m-1}\right)  - \frac{1}{\sigma}\right),1\right\},\\
\ql(q) &\eqdef \frac{m-1}{2} \erfc\left(\erfcinv\left(\frac{2q}{m-1}\right)  + \frac{1}{\sigma}\right),
\end{align*}

\newcommand{\stopnow}{\textsc{Stop}}
\newcommand{\continue}{\textsc{Continue}}
\newcommand{\return}{\textbf{return}\xspace}

\begin{algorithm}[!htb]
	\caption{\textbf{-- Local Sensitivity:} use the functions $\qu$ and $\ql$ to
		compute (an upper bound) of the local sensitivity at a given $q$ value by
		looking at the difference of $\ma{\cdot}$ evaluated on the bounds.}
	\label{alg:ls}
	\begin{algorithmic}[1] 		\Procedure{$\LS$}{$q$} 				\If{$q_1 \leq q \leq \qswitch$} \Comment{$q_1=\ql(\qswitch)$. Interpolate the middle part.}
		\State $q \gets q_1$
		\EndIf\label{}
		\State \textbf{return} $\max\{\ma{\qu(q)} - \ma{q}, \ma{q} - \ma{\ql(q)}\}$
		\EndProcedure
	\end{algorithmic}
\end{algorithm}

\newcommand{\atdistanced}{\textsc{AtDistanceD}}
\newcommand{\sort}{\textsc{Sort}}
\begin{algorithm}[!htb]
	\caption{\textbf{-- Sensitivity at a distance:} given a histogram $\barn$,
		compute the sensitivity of $\ma{\cdot}$ at distance at most $d$ using the procedure
		$\LS$, function $q(\cdot)$, constants $\qswitch$ and $q_1=\ql(\qswitch)$, and careful case analysis that finds the neighbor at distance $d$
		with the maximum sensitivity.}
	\label{alg:sens_at_k}
	\begin{algorithmic}[1]
		\Procedure{\atdistanced}{$\barn$, $d$}
		\State $q\gets q(\barn)$
		\If{$q_1\leq q \leq \qswitch$}\Comment{$q$ is in the flat region.}
		\State{\return $\LS(q)$, \stopnow}
		\EndIf
		\If{$q < q_1$} \Comment{Need to increase $q$.}
		\If{$n^{(1)} - n^{(2)} < 2d$}\Comment{$n^{(i)}$ is the $i$th largest element.}
		\State{\return $\LS(q_1)$, \stopnow}
		\Else
		\State{$\barn' \gets \sort(\barn) + [-d, d, 0, \dots, 0]$}
		\State{$q' \gets q(\barn')$}
		\If{$q' >q_1$}
		\State{\return $\LS(\qswitch)$, \stopnow}
		\Else
		\State{\return $\LS(q')$, \continue}
		\EndIf
		\EndIf
		\Else \Comment{Need to decrease $q$.}
		\If{$\sum_{i=2}^{d} n^{(i)} \leq d$}
		\State{$\barn' \gets [n, 0, \dots, 0]$}
		\State{$q' \gets q(\barn')$}
		\State{\return $\LS(q')$, \stopnow}
		\Else
		\State{$\barn' \gets \sort(\barn) + [d, 0, \dots, 0]$}
		\For{$d' = 1, \dots, d$}
		\State{$n^{\prime (2)} \gets n^{\prime (2)}-1$} \Comment{The index of $n^{\prime (2)}$ may change.}
		\EndFor	
		\State{$q' \gets q(\barn')$}
		\If{$q' < \qswitch$}
		\State{\return $\LS(\qswitch)$, \stopnow}
		\Else
		\State{\return $\LS(q')$, \continue}
		\EndIf
		\EndIf
		\EndIf
		\EndProcedure
	\end{algorithmic}
\end{algorithm}

\newcommand{\stopcondition}{\text{StoppingCondition}}

\begin{algorithm}[!htb]
	\caption{\textbf{-- Smooth Sensitivity:} Compute the $\beta$ smooth
		sensitivity of $\ma{\cdot}$ via \autoref{def:smooth-sensitivity} by
		looking at sensitivities at various distances and returning the maximum
		weighted by $e^{-\beta d}$.}
	\label{alg:ss}
	\begin{algorithmic}[1] 		\Procedure{SmoothSensitivity}{$\barn$, $\beta$}
		
		\State $S \gets 0$
		\State $d\gets 0$
		
		\Repeat
		\State{$c, \stopcondition \gets \atdistanced(\barn, d)$}
		\State{$S \gets \max\{S, c \cdot e^{-\beta d}\}$}
		\State{$d\gets d+1$}
		\Until{$\stopcondition=\stopnow$}
		\EndProcedure
	\end{algorithmic}
\end{algorithm}

\subsection{Notation and Conditions}

\paragraph{Notation.} We find that the algorithm and the proof of its correctness are more naturally expressed if we relax the notions of a histogram and its neighbors to allow non-integer values.
\begin{itemize}
\item We generalize histograms to be any vector with non-negative real values. This relaxation is used only in the analysis of algorithms; the actual computations are performed exclusively over integer-valued inputs.

\item Let $\barn=[n_1,\dots,n_m] \in \mathbb{R}^m$, $n_i\geq 0$ denote a histogram. Let $n^{(i)}$ denote the $i$-th bar in the descending order.
\item Define a ``move'' as increasing one bar by some value in $[0, 1]$ and decreasing one bar by a (possibly different) value in $[0, 1]$ subject to the resulting value be non-negative. Notice the difference between the original problem and our relaxation. In the original formulation, the histogram takes only integer values and we can only increase/decrease them by exactly $1$. In contrast, we allow real values and a teacher can contribute an arbitrary amount in $[0,1]$ to any one class.

\item Define the distance between two histograms $\barn = (n_1, \dots, n_m)$ and $\barn' = (n'_1, \dots, n'_m)$ as
\[
d(\barn, \barn') \eqdef 
\max\left\{
\sum_{i\colon n_i > n'_i} \lceil n_i - n'_i \rceil, \quad
\sum_{i\colon n_i < n'_i} \lceil n'_i - n_i \rceil
\right\},
\]
which is equal to the smallest number of ``moves'' needed to make the two histograms identical. We use the ceiling function since a single step can increase/decrease one bar by at most~$1$.

We say that two histograms are \emph{neighbors} if their distance $d$ is  1.

Notice that analyses of \renyi differential privacy for LNMax, GNMax and the
exponential mechanism are still applicable when the neighboring datasets are defined
in this manner. 

\item Given a randomized aggregator $\calM\colon \mathbb{R}_{\geq 0}^m\to\kset{m}$, let $q\colon \mathbb{R}_{\geq 0}^m \to [0, 1]$
be so that
\[
q(\barn)\geq \Pr[\calM(\barn)\neq \argmax(\barn)].
\]
When the context is clear, we use $q$ to denote a specific value of the function, which, in particular, can be used as $\qt$ in applications of \autoref{theorem:higher-to-lower-full}.

\item Let $\emptyma\colon [0, 1] \to \mathbb{R}$ be the function that maps a $q$ value to the value of the \renyi  accountant.
\end{itemize}

\paragraph{Conditions.} Throughout this section we will be referring to the list of conditions on $q(\cdot)$ and $\ma{\cdot}$:
\begin{itemize}
\item[C1.] The function $q(\cdot)$ is continuous in each argument $n_i$.

\item[C2.]
There exist functions $\qu,\ql\colon [0, 1] \to [0, 1]$ such that for any neighbor $\barn'$ of $\barn$, we have $\ql(q(\barn)) \leq q(\barn') \leq \qu(q(\barn))$, i.e., $\qu$ and $\ql$ provide upper and lower bounds on the $q$~value of any neighbor of $\barn$.

\item[C3.]  $\ql(q)$ is increasing in $q$.

\item[C4.] $\qu$ and $\ql$ are functional inverses of each other in part of the range, i.e.,  $q = \ql(\qu(q))$ for all $q \in [0, \qswitch]$, where $\qswitch$ is defined below. Additionally $\ql(q)\leq q\leq \qu(q)$ for all $q\in[0,1]$.

\item[C5.] $\ma{\cdot}$ has the following shape: there exist constants $\beta^*$ and 
  $\qswitch\leq 0.5$, such that $\ma{q}$ non-decreasing in $[0,\qswitch]$
    and $\ma{q}=\beta^*\geq \ma{\qswitch}$ for $q > \qswitch$. The constant  $\beta^*$ corresponds to a data-independent bound.
     
\item[C6.] $\Delta\ma{q} \eqdef \ma{\qu(q)} - \ma{q}$ is non-decreasing in $[0, \ql(\qswitch)]$, i.e., when $\qu(q) \leq \qswitch$.

\item[C7.] Recall that $n^{(i)}$ is the $i$-th largest coordinate of a histogram $\barn$. Then, if $q(\barn)\leq \qu(\qswitch)$, then $q(\barn)$ is differentiable in all coordinates and
\[
\forall i>j\geq 2\quad \pdv{q\hphantom{^{(i)}}}{n^{(j)}}(\barn) \geq \pdv{q\hphantom{^{(j)}}}{n^{(i)}}(\barn)\geq 0.
\]

\item[C8.] The function $q(\barn)$ is invariant under addition of a constant, i.e.,
\[
q(\barn) = q(\barn + [x,\dots,x])\textrm{ for all }\barn\textrm{ and }x\geq 0,
\]
and $q(\barn)$ is invariant under permutation of $\barn$, i.e.,
\[
q(\barn) = q(\pi(\barn))\textrm{ for all permutations }\pi\textrm{ on }\kset{m}.
\]
Finally, we require that if $n^{(1)}=n^{(2)}$, then $q(\barn)\geq \qswitch$.
\end{itemize}

We may additionally assume that  $\qswitch \geq q([n,0,\dots,0])$. Indeed, if this condition is not satisfied, then the data-dependent analysis is not going to be used anywhere. The most extreme histogram---$[n,0,\dots,0]$---is the most advantageous setting for applying data-dependent bounds. If we cannot use the data-dependent bound even in that case, we would be using the data-independent bound everywhere and do not need to compute smooth sensitivity anyway. Yet this condition is not automatically satisfied. For example, if $m$ (the number of classes) is large compared to $n$ (the number of teachers), we might have large $q([n,0,\dots,0])$. So we need to check this condition in the code before doing smooth sensitivity calculation.

\subsection{Correctness of Algorithms~\ref{alg:ls}--\ref{alg:ss}}

Recall that local sensitivity of a deterministic function $f$ is defined as $\max f(D)-f(D')$, where $D$ and $D'$ are neighbors.

\begin{proposition} \label{prop:alg_3_ls} Under conditions C2--C6, \autoref{alg:ls} computes an upper bound on local sensitivity of $\ma{q(\barn)}$.
\end{proposition}

\begin{proof}
Since $\ma{\cdot}$ is non-decreasing everywhere (by C5), and for any neighbors $\barn$ and $\barn'$ it holds that  $\ql(q(\barn)) \leq q(\barn') \leq \qu(q(\barn))$ (by C2), we have the following 
\begin{align*}
\left|\ma{q(\barn)}-\ma{q(\barn')}\right|  &\leq \max\bigg\{\emptyma\Big(\qu(q(\barn))\Big) -
  \emptyma\Big(q(\barn)\Big), \;\emptyma\Big(q(\barn)\Big) -
  \emptyma\Big(\ql(q(\barn))\Big)\bigg\}\\ &=
  \max\bigg\{\Delta\emptyma\Big(q(\barn)\Big),\;
  \Delta\emptyma\Big(\ql(q(\barn))\Big)\bigg\}
\end{align*}
as an upper bound on the local sensitivity of $\ma{q(\cdot)}$ at input $\barn$.

The function computed by \autoref{alg:ls} differs from above when $q(\barn)\in (\ql(\qswitch),
\qswitch)$. To complete the proof we need to argue that the local sensitivity is upper bounded by $\Delta\ma{\ql(\qswitch)}$ for $q(\barn)$ in this interval. The bound follows from the following three observations.

First, $\Delta\ma{q}$ is non-increasing in the range $(\ql(\qswitch), 1]$,  since $\ma{\qu(q)}$ is constant (by $\qu(q)\geq \qu(\ql(q_0))=q_0$ and C5) and $\ma{q}$ is non-decreasing in the range (by C5). In particular,
\begin{equation}\label{eq:q>=Bl(q0)}
\Delta\ma{q} \leq \Delta\ma{\ql(\qswitch)}\textrm{ if }q \geq \ql(\qswitch).
\end{equation}
Second, $\Delta\ma{\ql(q)}$ is non-decreasing in the range $[0, \qswitch]$ since $\ql(q)$ is increasing (by C3 and~C6). This implies that
\begin{equation}\label{eq:q<=q0}
\Delta\ma{\ql(q)} \leq \Delta\ma{\ql(\qswitch)}\textrm{ if }q \leq \qswitch.
\end{equation}

By \eqref{eq:q>=Bl(q0)} and \eqref{eq:q<=q0} applied to the intersection of the two ranges, it holds that
\[
\max\bigg\{\Delta\emptyma\Big(q(\barn)\Big),\;\Delta\emptyma\Big(\ql(q(\barn))\Big)\bigg\}\leq \Delta\ma{\ql(\qswitch)}\textrm{ if } \ql(\qswitch)\leq q\leq \qswitch,
\]
as needed.
\end{proof}

We thus established that the function computed by \autoref{alg:ls}, which we call $\LS(q)$ from now on, is an upper bound on the local sensitivity. Formally, 
\begin{align*}
\LS(q) \eqdef
\begin{cases}
\Delta\ma{\ql(\qswitch)} & \text{ if } q\in (\ql(\qswitch),
\qswitch),\\
\max\left\{\Delta\ma{q},
\Delta\ma{\ql(q)}\right\} & \text{ otherwise}.
\end{cases}
\end{align*}

The following proposition characterizes the growth of $\LS(q)$.

\begin{proposition}\label{prop:growth_ls}
Assuming conditions C2--C6, the function $\LS(q)$ is non-decreasing in $[0, \ql(\qswitch)]$, constant in $[\ql(\qswitch), \qswitch]$, and non-increasing in $[\qswitch, 1]$.
\end{proposition}

\begin{proof}
Consider separately three intervals.
\begin{itemize}
\item By construction, $\LS$ is constant in $[\ql(\qswitch), \qswitch]$. 

\item Since both functions $\Delta\ma{\cdot}$ and $\Delta\ma{\ql(\cdot)}$ are each non-decreasing in $[0, \ql(\qswitch))$, so is their max.

\item In the interval $(\qswitch, 1]$, $\ma{q}$ is constant. Hence $\Delta\ma{q}=0$ and $\Delta\ma{\ql(q)}=\ma{q}-\ma{\ql(q)}$ is non-decreasing. Their maximum value $\Delta\ma{\ql(q)}$ is non-decreasing.
\end{itemize}
The claim follows.
\end{proof}

We next prove correctness of~\autoref{alg:sens_at_k}, which computes the maximal sensitivity of \emptyma\ at a fixed distance.

The proof relies on the following notion of a partial order between histograms.

\begin{definition}\label{def:dominance}
Prefix sums $S_i(\barn)$ are defined as follows:
\[
S_i(\barn)\eqdef\sum_{j=1}^{i} (n^{(1)}-n^{(j)}). 
\]
We say that a histogram $\barn$ \emph{dominates} $\barn'$, denoted as $\barn\succeq\barn'$, iff:
\[
  \forall \left(1 < i \leq m\right)\textrm{ it holds that } S_i(\barn) \geq S_i(\barn').
\]
\end{definition}

The function $q(\cdot)$ is monotone under this notion of dominance (assuming certain conditions hold):
\begin{proposition}\label{prop:monotonicity}If $q(\cdot)$ satisfies C1, C2, C7, and C8, and $q(\barn)<\qu(\qswitch)$, then
\[
\barn\succeq\barn'\Rightarrow q(\barn)\leq q(\barn').
\]
\end{proposition}
\begin{proof}
We may assume  that $n^{(1)}= n^{\prime(1)}$. Indeed, if this does not hold, add $|n^{(1)}-n^{\prime(1)}|$ to all coordinates of the histogram with the smaller of the two values. This transform does not change the $q$ value (by C8) and it preserves the $\succeq$ relationship as all prefix sums $S_i(\cdot)$ remain unchanged.

We make a simple observation that will be helpful later:
\begin{equation}\label{eq:unsorted_prefix_sum}
\forall i \in \kset{m}\textrm{ it holds that } \sum_{j=1}^{i} (n^{(1)}-n_j)\geq S_{i}(\barn).
\end{equation}
The inequality holds because the prefix sum accumulates the gaps between the largest value of $\barn$ and all other values in the non-decreasing order. Any deviation from this order may only increase the prefix sums.

The following lemma constructs a monotone chain (in the partial order of dominance) of histograms connecting $\barn$ and $\barn'$ via a sequence of intermediate steps that either do not change the value of $q$ or touch at most two coordinates at a time. 

\begin{lemma}\label{lemma:chain}There exists a chain $\barn=\barn_0\succeq\barn_1\succeq\cdots\succeq\barn_d=\barn'$, such that for all $i\in[d]$ either $d(\barn_{i-1},\barn_i)=1$ or $\barn_{i-1}=\pi(\barn_i)$ for some permutation $\pi$ on $\kset{m}$. Additionally, $n^{(1)}_0=\dots=n^{(1)}_d$.
\end{lemma}
\begin{proof}[Proof (Lemma)]
Wlog we assume that $\barn$ and $\barn'$ are each sorted in the descending order. The proof is by induction on $\ell(\barn,\barn')\eqdef\sum_i \lceil n_i-n'_i\rceil\leq 2d(\barn,\barn')$, which, by construction, only assumes non-negative integer values. 

If the distance is 0, the statement is immediate. Otherwise, find the smallest
  $i$ so that $S_i(\barn)>S_i(\barn')$ (if all prefix sums are equal and
  $n^{(1)}= n^{\prime(1)}$, it would imply that $\barn=\barn'$). In
  particular, it means that $n_j=n'_j$ for $j<i$ and $n_i<n'_i\leq n_{i-1}=n'_{i-1}$. Let $x\eqdef \min(n'_i-n_i,1)$. Define $\barn''$ as identical to $\barn'$ except that $n''_i=n'_i-x$. The new value is guaranteed to be non-negative, since $x\leq n'_i-n_i$ and $n_i\geq 0$. Note that $\barn''$ is not necessarily sorted anymore. Consider two possibilities.

\lead{Case I: $\barn\succeq\barn''$.} Since $\barn''\succeq \barn'$,  $\ell(\barn,\barn'')<\ell(\barn,\barn')$, and $d(\barn'', \barn')=1$, we may apply the induction hypothesis to the pair $\barn$, $\barn''$.

\lead{Case II: $\barn\nsucceq\barn''$.} This may happen because the prefix sums of $\barn''$ increase compared to $S_j(\barn')$ for $j\geq i$. Find the smallest such $i'$ so that $\sum_{j=1}^{i'} (n''_1-n''_j)>S_{i'}(\barn)$. (Since $\barn''$ is not sorted, we fix the order in which prefix sums are accumulated to be the same as in $\barn$; by (\ref{eq:unsorted_prefix_sum}) $i'$ is well defined). Next we let $\barn'''$ be identical to $\barn''$ except that $n_{i'}'''=n_{i'}''+x$. In other words, $\barn'''$ differs from $\barn'$ by shifting $x$ from coordinate $i$ to coordinate $i'$. 

We argue that incrementing $n_{i'}''$ by $x$ does not change the maximal value of $\barn''$, i.e., $n_1'''>n_{i'}'''$.  Our choice of $i'$, which is the smallest index so that the prefix sum over $\barn''$ overtakes that over $\barn$, implies that $n_1''-n_{i'}''>n_1-n_{i'}$. Since $n_1''=n_1$, it means that $n_{i'}>n_{i'}''$ (and by adding $x$ we move $n_{i'}''$ towards $n_{i'}$).  Furthermore,
\[
n_{i'}'''=n_{i'}''+x\leq n_{i'}+(n_i'-n_i)=n'_i+(n_{i'}-n_i)\leq n'_i\leq  n'_1=n'''_1,
\]
(We use $n_{i'}\leq n_i$, which is implied by $i'>i$.)

We claim that $\sum_{j=1}^t (n_1'''-n_j''')\leq S_t(\barn)$ for all $t$, and thus, via (\ref{eq:unsorted_prefix_sum}), $\barn\succeq\barn'''$. The choice of $i'$ makes the statement  trivial for $t<i'$. For $t\geq i'$ the following holds:
\[
\sum_{j=1}^t (n_1'''-n_j''')=\sum_{j=1}^t (n_1''-n_j'')-x\leq \left(\sum_{j=1}^t (n_1'-n_j')+x\right)-x=S_t(n')\leq S_t(\barn).
\]

By construction $d(\barn',\barn''')=1$ (the two histograms differ in two locations, in positive and negative directions, by $x\leq 1$ in each). For the same reasons $\barn'''\succeq\barn'$. To show that $\ell(\barn,\barn')>\ell(\barn,\barn''')$, compare  $\lceil n_j-n'_j\rceil$ and  $\lceil n_j-n'''_j\rceil$ for $j=i,i'$. At $j=i$ the first term is strictly larger than the second. At $j=i'$, the inequality holds too but it may be not strict.

We may again apply the induction hypothesis to the pair $\barn$ and $\barn'''$, thus completing the proof of the lemma.
\end{proof}	

To complete the proof of the proposition, we need to argue that the values of $q$ are also monotone in the chain constructed by the previous lemma.  Concretely, we put forth
\begin{lemma}\label{lemma:distance1}If $\barn'\succeq\barn$, $q(\barn')\leq \qu(\qswitch)$, $d(\barn,\barn')=1$ and $\barn^{(1)}=\barn^{\prime(1)}$, then $q(\barn)\leq q(\barn')$.
\end{lemma}
\begin{proof}
The fact that $d(\barn,\barn')=1$ and $\barn\succeq\barn'$ means that there is either a single index $i$ so that $n_i'<n_i$, or there exist two indices $i$ and $j$ so that $n'_i<n_i$ and $n'_j>n_j$. The first case is immediate, since $q$ is non-decreasing in all inputs except for the largest (by C7).

Let $n'_i=n_i-x$ and $n'_j=n_j+y$, where $x,y>0$. Since $\barn\succeq\barn'$, it follows that $n_i\geq n_j$ and $x>y$. Consider two cases.

\lead{Case I: $n'_i\geq n'_j$}, i.e., removing $x$ from $n_i$ and adding $y$ to $n_j$ does not change their ordering. Let
\[
\barn(t)\eqdef (1-t)\barn+t\cdot \barn'=[n_1,\dots,n_i-t\cdot x,\dots,n_j+t\cdot y,\dots,n_m].
\]
Then,
\begin{align*}
q(\barn')-q(\barn)=q(\barn(1))-q(\barn(0))&=\int_{t=0}^1 (q\circ \barn)'(t)\,\mathrm{d}t\\
&=\int_{t=0}^1\left\{-x\pdv{q}{n_i}\barn(t)+y\pdv{q}{n_j}\barn(t)\right\}\,\mathrm{d}t\\
&\leq 0. 
\end{align*}
The last inequality follows from C7 and the facts that $x>y>0$ and $n_i(t)>n_j(t)$. (The condition that $q(\barn(t))\leq \qu(\qswitch)$ follows from C2 and the fact that $d(\barn',\barn(t))\leq 1$.)

\lead{Case II: $n'_i\leq n'_j$.} In this case we swap the $i$th and $j$th indices in $\barn'$ by defining $\barn''$ which differs from it in $\barn_i''=\barn'_j$ and $\barn_j''=\barn'_i$. By C8, $q(\barn'')=q(\barn')$ and, of course, $\barn''\succeq\barn$ since the prefix sums remain unchanged. The benefit of doing this transformation is that we are back in Case I, where the relative order of coordinates that change between $\barn$ and $\barn''$ remains the same.

This concludes the proof of the lemma.
\end{proof}

Applying \autoref{lemma:chain} we construct a chain of histograms between $\barn$ and $\barn'$, which, by \autoref{lemma:distance1}, is non-increasing in $q(\cdot)$. Together this implies that $q(\barn)\leq q(\barn')$, as claimed.
\end{proof}

We apply the notion of dominance in proving the following proposition, which is used later in arguing correctness of \autoref{alg:sens_at_k}.

\begin{proposition}\label{prop:alg4}Let $\barn$ be an integer-valued histogram and $d$ be a positive integer. And $q(\cdot)$ satisfies C1, C7, and C8. The following holds:
\begin{enumerate}
	\item Assuming $n^{(1)}-n^{(2)}\geq 2d$, let $\barn^*$ be
	  obtained from $\barn$  by decrementing $n^{(1)}$ by $d$ and incrementing
	  $n^{(2)}$ by $d$. Then
	\[
	d(\barn,\barn^*)=d\textrm{ and }q(\barn^*)\geq q(\barn')\textrm{ for any }\barn'\textrm{ such that  }d(\barn,\barn')=d.
	\]
	\item Assuming $\sum_{i=2}^m n^{(i)}\geq d$, let $\barn^{**}$ be obtained
	  from $\barn$ by incrementing $n_1$ by $d$, and by repeatedly decrementing
	  the histogram's  \emph{current} second highest value by one, $d$ times. Then
		\[
	d(\barn,\barn^{**})=d\textrm{ and }q(\barn^{**})\leq q(\barn')\textrm{ for any }\barn'\textrm{ such that  }d(\barn,\barn')=d.
	\]
\end{enumerate}
\end{proposition}

\begin{proof}Towards proving the claims, we argue that $\barn^*$ and $\barn^{**}$ are, respectively, the minimal and the maximal elements in the histogram dominance order (\autoref{def:dominance}) in the set of histograms at distance $d$ from $\barn$. By \autoref{prop:monotonicity} the claims follow.
\begin{enumerate}
\item Take any histogram $\barn'$ at distance $d$ from $\barn$. Our goal is to prove that $\barn'\succeq\barn^*$. Recall the definition of the distance $d(\cdot,\cdot)$ between two histograms $d(\barn, \barn')=\max\left\{\sum_{i\colon n_i > n'_i} \lceil n_i - n'_i \rceil, \sum_{i\colon n_i < n'_i} \lceil n'_i - n_i \rceil \right\}$. If the distance is bounded by $d$, it means, in particular, that
\begin{align*}
n^{\prime(1)}&\geq n^{(1)}-d\textrm{ and }\\
\sum_{j=2}^i n^{\prime(j)}&\leq \sum_{j=2}^i n^{(j)}+d\quad\textrm{ for all }i>2.
\end{align*}
That lets us bound the prefix sums of $\barn'$ as follows:
\begin{multline*}
	S_i(\barn')=\sum_{j=2}^{i} (n^{\prime(1)}-n^{\prime(j)})=(i-1)\cdot n^{\prime(1)}-\sum_{j=2}^{i} n^{\prime(j)}\\
\geq(i-1)\cdot (n^{(1)}-d)-\left(\sum_{j=2}^{i} n^{(j)}+d\right)=S_i(\barn^*).
\end{multline*}
We demonstrated that $\barn'\succeq\barn^*$, which, by \autoref{prop:monotonicity}, implies that $q(\barn')\leq q(\barn^*)$. Together with the immediate $d(\barn,\barn^*)=d$ we prove the claim.

\item Assume wlog that $\barn$ is sorted in the descending order. Define the following value that depends on $\barn$ and $d$:
\[
u\eqdef \min\left\{x\in \mathbb{N}\colon \sum_{i\colon i>1, n_i\geq x} n_i -x \leq d\right\}.
\]
The constant $u$ is the smallest such $x$ so that the total mass that can be shaved from elements of $\barn$ above $x$ (excluding $n_1$) is at most $d$.

We give the following equivalent definition of $\barn^{**}$:
\[
n_{i}^{**}=
\begin{cases}
n_1+d&\textrm{ if }i=1,\\
u&\textrm{ if } i>1\textrm{ and }n_i\geq u,\\
n_i&\textrm{ otherwise.}
\end{cases}
\]
Fix any $i\in\kset{m}$ and any histogram $\barn'$ at distance $d$ from $\barn$. Our goal is to prove that $S_i(\barn^{**})\geq S_i(\barn')$ and thus $\barn^{**}\succeq\barn'$. Assume the opposite and take largest $i$ such that $S_i(\barn^{**})< S_i(\barn')$.

We may assume that  $n^{\prime(1)}=n_1^{**}=n_1+d$. Consider the following cases.

\lead{Case I.} If $n^{**(i)}<u$,  the contradiction follows from
\begin{multline*}
S_i(\barn')=\sum_{j=2}^{i} (n^{\prime(1)}-n^{\prime(j)})=\sum_{j=2}^{i} \left((n^{\prime(1)}-n^{(1)})+(n^{(1)}-n^{(j)})+(n^{(j)}-n^{\prime(j)})\right)\\
\leq (i-1)d+S_i(\barn)+d=S_i(\barn^{**}).
\end{multline*}
The last equality is due to the fact that all differences between $\barn$ and $\barn^{**}$ are confined to the indices that are less than $i$.

\lead{Case II.} If $n^{**(i)}=u$ and $n^{\prime(i)}\geq u$, the contradiction with $S_i(\barn^{**})<S_i(\barn')$ follows immediately from
\[
S_i(\barn')=\sum_{j=2}^{i} (n^{\prime(1)}-n^{\prime(j)})\leq (i-1)(n^{\prime(1)}-u)=S_i(\barn^{**}).
\]

\lead{Case III.} Finally, consider the case when  $n^{**(i)}=u$ and $v\eqdef n^{\prime(i)}< u$.  Since $i$ is the largest such that $S_i(\barn^{**})< S_i(\barn')$, it means that $n^{**(i+1)}<n^{\prime(i+1)}\leq v<u=n^{**(i)}$ and thus  $n^{**(i)}-n^{**(i+1)}\geq 2$ (we rely on the fact that the histograms are integer-valued). It implies that all differences between $\barn$ and $\barn^{**}$ are confined to the indices in $[1, i]$. Then,
\begin{align*}
S_i(\barn^{**})- S_i(\barn')&\geq \sum_{j=2}^i (n_1^{**}-n_j^{**})- \sum_{j=2}^i (n'_1-n'_j)\tag{by (\ref{eq:unsorted_prefix_sum})}\\
&=\sum_{j=2}^i ((n_j-n_j^{**})+(n'_j-n_j))\tag{since $n_1^{**}=n'_1$}\\
&\geq d-d(\barn,\barn')\\
&\geq 0,
\end{align*}
which contradicts the assumption that $S_i(\barn^{**})<S_i(\barn')$.
\end{enumerate}
\end{proof}

We may now state and prove the main result of this section.

\begin{theorem}\label{thm:alg_ss_correctness}
Assume that $q(\cdot)$ satisfies conditions C1--C8 and $\barn$ is an integer-valued histogram. Then the following two claims are true:
\begin{enumerate}
	\item  \autoref{alg:sens_at_k} computes	$\max_{\barn'\colon \dist(\barn, \barn') \leq d} \LS(\barn')$ .\\
	\item  \autoref{alg:ss} computes $\SS_\beta(\barn)$, which is a $\beta$-smooth upper bound on smooth sensitivity of~$\ma{q(\cdot)}$.
\end{enumerate}
\end{theorem}
\begin{proof}
\lead{Claim 1.} Recall that $q_1=\ql(\qswitch)$, and therefore, by  \autoref{prop:growth_ls}  the function $\LS(q)$ is non-decreasing in $[0, q_1]$, constant in $[q_1, \qswitch]$, and non-increasing in $[\qswitch, 1]$.  It means, in particular, that to maximize $\LS(q(\barn'))$ over histograms satisfying $d(\barn,\barn')=d$, it suffices to consider the following cases.
	
	If $\LS(q(\barn))<q_1$, then higher values of $\LS(\cdot)$ may be attained only by histograms with higher values of $q$. \autoref{prop:alg4} enables us to efficiently find a histogram $\barn^*$ with the highest $q$ at distance $d$, or conclude that we may reach the plateau by making the two highest histogram entries be equal.
	
	If $q_1\leq \LS(q(\barn))\leq \qswitch$, it means that $\LS(q(n))$ is already as high as it can be.
	
	If $\qswitch<\LS(q(\barn))$, then, according to \autoref{prop:growth_ls}, higher values of $\LS(\cdot)$ can be achieved by histograms with smaller values of $q$, which we explore using the procedure outlined by \autoref{prop:alg4}. The stopping condition---when the plateau is reached---happens when $q$ becomes smaller than $\qswitch$.
	
\lead{Claim 2.} The second claim follows from the specification of \autoref{alg:ss} and the first claim.
\end{proof}

\subsection{GNMax Satisfies Conditions C1--C8}\label{ss:gaussian-satisfies-conditions}

The previous sections laid down a framework for computing smooth sensitivity of a randomized aggregator mechanism: defining functions $q(\cdot)$, $\qu(\cdot)$, $\ql(\cdot)$, verifying that they satisfy conditions C1--C8, and applying \autoref{thm:alg_ss_correctness}, which asserts correctness of \autoref{alg:ss}. In this section we instantiate this framework for the GNMax mechanism.

\subsubsection{Conditions C1--C4, C7 and C8}

\paragraph{Defining $q$ and conditions C1, C7, and C8.} Following \autoref{prop:gaussian-q-full}, we define $q\colon \mathbb{R}_{\geq 0}^m\to [0,1]$ for  a GNMax mechanism parameterized with $\sigma$ as:
\begin{align*}
	q(\barn) &\eqdef \min\left\{\sum_{i\neq i^*} \Pr(Z_i - Z_{i^*} \geq n_{i^*} - n_i), 1\right\}\\ 
	&= \min\left\{\sum_{i\neq i^*} \frac{1}{2} \left(1 - \erf\left(\frac{n_{i^*} - n_i}{2\sigma}\right)\right), 1\right\}\\
	&= \min\left\{\sum_{i\neq i^*} \frac{1}{2} \erfc\left(\frac{n_{i^*} - n_i}{2\sigma}\right), 1\right\},
\end{align*}
where $i^*$ is the histogram $\barn$'s highest coordinate, i.e., $n_{i^*} \geq n_i$ for all $i$ (if there are multiple highest,  let $i^*$ be any of them). Recall that $\erf$ is the error function, and $\erfc$ is the complement error function.

\autoref{prop:gaussian-q-full} demonstrates that $q(\barn)$ bounds from above the probability that GNMax outputs anything but the highest coordinate of the histogram.

Conditions C1, C7, and C8 follow by simple calculus ($\qswitch$, defined
below, is at most 0.5).

\paragraph{Functions $\ql$, $\qu$, and conditions C2--C4.}  Recall that the functions $\ql$ and $\qu$ are defined in \autoref{ap:smooth-sensitivity} as follows:
\begin{align*}
\qu(q) &\eqdef \min\left\{\frac{m-1}{2} \erfc\left(\erfcinv\left(\frac{2q}{m-1}\right)  - \frac{1}{\sigma}\right),1\right\},\\
\ql(q) &\eqdef \frac{m-1}{2} \erfc\left(\erfcinv\left(\frac{2q}{m-1}\right)  + \frac{1}{\sigma}\right),
\end{align*}

\begin{proposition}[Condition C2]  For any neighbor $\barn'$ of $\barn$, i.e., $d(\barn',\barn)=1$, the following bounds hold:
	\[
	\ql(q(\barn))\leq q(\barn')\leq \qu(q(\barn)).
	\]
\end{proposition}
\begin{proof}Assume wlog that $i^*=1$. Let $x_i\eqdef n_1-n_i$ and $q_i\eqdef \erfc(x_i/2\sigma)/2$, and similarly define $x'_i$ for $\barn'$. Observe that $|x_i-x_i'|\leq 2$, which, by monotonicity of $\erfc$, implies that
\[
 \frac12 \erfc\left(\frac{x_i+2}{2\sigma}\right)\leq q_i(\barn') \leq \frac12 \erfc\left(\frac{x_i-2}{2\sigma}\right).
\] 
Thus
\[
\frac12 \sum_{i>1} \erfc\left(\frac{x_i + 2}{2\sigma}\right) \leq q(\barn') \leq \frac12 \sum_{i>1} \erfc\left(\frac{x_i - 2}{2\sigma}\right).
\]
(Although $i^*$ may change between $\barn$ and $\barn'$, the bounds still hold.)

Our first goal is to upper bound $q(\barn')$ for a given value of $q(\barn)$. To this end we set up the following maximization problem
\[
\max_{\{x_i\}} \frac12 \sum_{i>1} \erfc\left(\frac{x_i - 2}{2\sigma}\right) \text{ such that } \frac12 \sum_{i>1} \erfc\left(\frac{x_i}{2\sigma}\right)= q\textrm{ and }x_i\geq 0.
\]
We may temporarily ignore the non-negative constraints, which end up being satisfied by our solution. Consider using the method of Lagrange multipliers and take a derivative in $x_i$':
\begin{align*}
-\exp\left(-\power{\frac{x_i-2}{2\sigma}}{2}\right) + \order \exp\left(-\power{\frac{x_i}{2\sigma}}{2}\right) &= 0\\
\Leftrightarrow \order &= \exp\left(\frac{x_i-1}{\sigma^2}\right).
\end{align*}
Since the expression is symmetric in $i>1$, it means that the local optima are attained at $x_2 = \dots = x_m$ (the second derivative confirms that these are local maxima). After solving for $(m-1)\erfc(x/2\sigma)=2q$ we have
\[
q(\barn') \leq \frac{m-1}{2} \erfc\left(\erfcinv\left(\frac{2q}{m-1}\right)  - \frac{1}{\sigma}\right).
\]
where $m$ is the number of classes.
Similarly, 
\[
q(\barn') \geq  \frac{m-1}{2} \erfc\left(\erfcinv\left(\frac{2q}{m-1}\right)  + \frac{1}{\sigma}\right).
\]
\end{proof}

Conditions C3, i.e., $\ql(q)$ is monotonically increasing in $q$, and C4, i.e., $\ql$ and $\qu$ are functional inverses of each other in $[0,\qswitch]$ and $\ql(q)\leq q\leq \qu(q)$ for all $q\in[0,1]$, follow from basic properties of $\erfc$. The restriction that $q\in[0,\qswitch]$ ensures that $\qu(q)$ is strictly less than one, and the minimum in the definition of $\qu(\cdot)$ simplifies to its first argument in this range.

\subsubsection{Conditions C5 and C6}
Conditions C5 and C6 stipulate that the function $\ma{q}\eqdef\MA_\sigma(q)$
(defined in \autoref{ap:privacy-analysis}) exhibits a specific growth pattern. Concretely, C5 states that $\ma{q}$ is monotonically increasing for $0\leq q\leq \qswitch$, and constant for $\qswitch<q\leq 1$. (Additionally, we require that $\qu(\qswitch)<1$). Condition C6 requires that $\Delta\ma{q} = \ma{\qu(q)} - \ma{q}$ is non-decreasing in $[0, \ql(\qswitch)]$.

Rather than proving these statements analytically, we check these assumptions for any fixed $\sigma$ and~\order\ via a combination of symbolic and numeric analyses.

More concretely, we construct symbolic expressions for $\ma{\cdot}$ and $\Delta\ma{\cdot}$ and (symbolically) differentiate them. We then minimize (numerically) the resulting expressions over $[0, \qswitch]$ and $[0, \ql(\qswitch)]$, and verify that their minimal values are indeed non-negative.

\subsection{\renyi Differential Privacy and Smooth Sensitivity}\label{ss:rdp_and_ss}

Although the procedure for computing a smooth sensitivity bound may be quite involved (such as Algorithms~\ref{alg:ls}--\ref{alg:ss}), its use in a differentially private data release is straightforward. Following \cite{nissim2007smooth}, we define an additive Gaussian mechanism where the noise distribution is scaled by $\sigma$ \textbf{and} a smooth sensitivity bound:
\begin{definition}\label{def:GNSS}Given a real-valued function $f\colon\calD\to\bbR$ and a $\beta$-smooth sensitivity bound $\SS(\cdot)$, let $(\beta,\sigma)$-GNSS mechanism be 
\[
\calF_\sigma(D)\eqdef f(D)+\SS_\beta(D)\cdot\calN(0,\sigma^2).
\]
\end{definition}

We claim that this mechanism satisfies \renyi differential privacy for finite orders from a certain range.
\begin{theorem}\label{prop:rdp_of_GNSS}The $(\beta,\sigma)$-GNSS mechanism $\calF_\sigma$ is $(\order, \eps)$-RDP, where
\[
\eps\eqdef \frac{\order\cdot e^{2\beta}}{\sigma^2}+\frac{\beta\order-0.5\ln(1-2\order\beta)}{\order-1}
\]
for all $1<\order<1/(2\beta)$.
\end{theorem}
\begin{proof}
Consider two neighboring datasets $D$ and $D'$. The output distributions of the $(\beta,\sigma)$-GNSS mechanism on $D$ and $D'$ are, respectively, 
\[
P\eqdef f(D)+\SS_\beta(D)\cdot\calN(0,\sigma^2) = \calN(f(D),(\SS_\beta(D)\sigma)^2) \textrm{ and } Q\eqdef \calN(f(D'),(\SS_\beta(D')\sigma)^2).
\]
The \renyi divergence between two normal distributions can be computed in closed form \citep{EH07-Renyi}:
\begin{equation}\label{eq:renyi_normals}
D_\order(P\|Q)=\order\frac{(f(D)-f(D'))^2}{2\sigma^2 s^2}+\frac{1}{1-\order}\ln\frac{s}{\SS_\beta(D)^{1-\order}\cdot \SS_\beta(D')^\order},
\end{equation}
provided $s^2\eqdef (1-\order)\cdot \SS_\beta(D)^2+\order\cdot\SS_\beta(D')^2>0$. 

According to the definition of smooth sensitivity (\autoref{def:smooth-sensitivity})
\begin{gather}
e^{-\beta}\cdot\SS_\beta(D)\leq \SS_\beta(D')\leq e^{\beta}\cdot\SS_\beta(D),\label{eq:bound_on_ss}\\
\intertext{and}
|f(D)-f(D')|\leq e^{\beta}\cdot\min(\SS_\beta(D),\SS_\beta(D')).\label{eq:bound_on_ls}
\end{gather}

Bound \eqref{eq:bound_on_ss} together with the condition that $\order\leq 1/(2\beta)$ implies that
\begin{multline}\label{eq:s_squared_bound}
s^2=(1-\order)\cdot \SS_\beta(D)^2+\order\cdot\SS_\beta(D')^2=\SS_\beta(D)^2+\order(\SS_\beta(D')^2-\SS_\beta(D)^2)\\
\geq \SS_\beta(D)^2(1+\order(e^{-2\beta}-1))\geq \SS_\beta(D)^2(1-2\order\beta)>0.
\end{multline}
The above lower bound ensures that $s^2$ is well-defined, i.e., non-negative, as required for application of~\eqref{eq:renyi_normals}.

Combining  bounds \eqref{eq:bound_on_ss}--  \eqref{eq:s_squared_bound}, we have that 
\[
D_\order(P\|Q)\leq \frac{\order\cdot
  e^{2\beta}}{\sigma^2}+\frac{1}{1-\order}\ln\left\{\frac{s}{\SS_\beta(D)}
  e^{-\order \beta}\right\}\leq\frac{\order\cdot e^{2\beta}}{\sigma^2}+\frac{\beta\order-0.5\ln(1-2\order\beta)}{\order-1}
\]
as claimed.
\end{proof}

Note that if $\lambda\gg 1$, $\sigma\ll \lambda$, and $\beta \ll 1/(2\lambda)$, then $(\beta,\sigma)$-GNSS satisfies $(\lambda, (\lambda+1)/\sigma^2)$-RDP. Compare this with RDP analysis of the standard additive Gaussian mechanism, which satisfies $(\lambda, \lambda/\sigma^2)$-RDP. The difference is that GNSS scales noise in proportion to \emph{smooth sensitivity}, which is no larger and can be much smaller than global sensitivity.

\subsection{Putting It All Together: Applying Smooth Sensitivity}

Recall our initial motivation for the smooth sensitivity analysis: enabling
privacy-preserving release of data-dependent privacy guarantees. Indeed, these
guarantees vary greatly between queries (see \autoref{fig:threshold-check}) and
are typically much smaller than data-independent privacy bounds. Since
data-dependent bounds may leak information about underlying data, publishing
the bounds themselves requires a differentially private mechanism. As we  explain shortly, smooth sensitivity analysis is a natural fit for this task.

We first consider the standard additive noise mechanism where the noise (such
as Laplace or Gaussian) is calibrated to the global sensitivity of the
function we would like to make differentially private.  We know that \renyi
differential privacy is additive for any fixed order \order, and thus the
cumulative RDP cost is the sum of RDP costs of individual queries each upper
bounded by a data-independent bound. Thus, it might be tempting to use the
standard additive noise mechanism for sanitizing the total, but that would be a
mistake.

To see why, consider a sequence of queries $\barn_1,\dots,\barn_\ell$ answered by the aggregator. Their total (unsanitized) RDP cost of order \order\ is $B_\sigma=\sum_{i=1}^\ell \MA_\sigma(q(\barn_i))$. Even though $\MA_\sigma(q(\barn_i ))\leq \order/\sigma^2$ (the data-independent bound, \autoref{prop:gaussian_rdp}), the sensitivity of their sum is \emph{not} $\order/\sigma^2$. The reason is that the (global) sensitivity is defined as the maximal difference in the function's output between two neighboring datasets $D$ and $D'$. Transitioning from $D$ to $D'$ may change one teacher's output on \emph{all} student queries.

In contrast with the global sensitivity of $B_\sigma$ that may be quite
high---particularly for the second step of the Confident GNMax aggregator---its smooth sensitivity can be extremely small. Towards computing a smooth sensitivity bound on $B_\sigma$, we prove the following theorem which defines a smooth sensitivity of the sum in terms of local sensitivities of its parts.

\begin{theorem}\label{thm:ss_of_sum}Let $f_i\colon\calD\to\mathbb{R}$ for $1\leq i\leq \ell$, $F(D)\eqdef \sum_{i=1}^\ell f_i(D)$ and $\beta>0$. Then 
\[
\SS(D)\eqdef \max_{d\geq 0} e^{-\beta d} \cdot \sum_{i=1}^\ell \max_{D'\colon \dist(D, D') \leq  d} \LS_{f_i}(D'),
\]
is a $\beta$-smooth bound on $F(\cdot)$ if $\LS_{f_i}(D')$ are upper bounds on the local sensitivity of $f_i(D')$.
\end{theorem}
\begin{proof}We need to argue that $\SS(\cdot)$ is $\beta$-smooth, i.e., $\SS(D_1)\leq e^\beta\cdot \SS(D_2)$ for any neighboring $D_1,D_2\in\calD$, and it is an upper bound on the local sensitivity of $F(D_1)$, i.e., $\SS(D_1)\geq \left|F(D_1)-F(D_2)\right|$.
	
Smoothness follows from the observation that
\[
\max_{D\colon \dist(D_1, D) \leq  d} \LS_{f_i}(D)\leq \max_{D\colon \dist(D_2, D) \leq  d+1} \LS_{f_i}(D)
\]
for all neighboring datasets $D_1$ and $D_2$ (by the triangle inequality over distances). Then
\begin{align*}
\SS(D_1)&= \max_{d\geq 0} e^{-\beta d} \cdot \sum_{i=1}^\ell \max_{D\colon \dist(D_1, D) \leq  d} \LS_{f_i}(D),\\
&\leq \max_{d\geq 0} e^{-\beta d} \cdot \sum_{i=1}^\ell \max_{D\colon \dist(D_2, D) \leq  d+1} \LS_{f_i}(D)\\
&= \max_{d'\geq 1} e^{-\beta (d'-1)} \cdot \sum_{i=1}^\ell \max_{D\colon \dist(D_2, D) \leq  d'} \LS_{f_i}(D)\\
&\leq e^{\beta}\cdot \SS(D_2)
\end{align*}
as needed for $\beta$-smoothness.

The fact that $\SS(\cdot)$ is an upper bound on the local sensitivity of $F(\cdot)$ is implied by the following:
\begin{align*}
\left|F(D_1)-F(D_2)\right|&=\left|\sum_{i=1}^\ell f_i(D_1)-\sum_{i=1}^\ell f_i(D_2)\right|\\
&\leq \sum_{i=1}^\ell \left|f_i(D_1)-f_i(D_2)\right|\\
&\leq \sum_{i=1}^\ell \LS_{f_i}(D_1)\\
&\leq \SS(D_1),
\end{align*}
which concludes the proof.
\end{proof}

Applying \autoref{thm:ss_of_sum} allows us to compute a smooth sensitivity of the sum more efficiently than summing up smooth sensitivities of its parts. Results below rely on this strategy.

\paragraph{Empirical results.} \autoref{table:smooth_sensitivity} revisits the privacy bounds in
\autoref{table:results_summary}. For all  data-dependent privacy claims of the Confident GNMax aggregator we report parameters for their smooth sensitivity analysis and results of applying the GNSS mechanism for their release.

Consider the first row of the table. The MNIST dataset was partitioned among
250 teachers, each getting 200 training examples. After the teachers were
individually trained, the student selected at random 640 unlabeled examples,
and submitted them to the Confident GNMax aggregator with the threshold of 200, and
noise parameters $\sigma_1=150$ and $\sigma_2=40$. The expected number of
answered examples (those that passed the first step of \autoref{alg:confident
aggregator}) is 283, and the expected \renyi differential privacy is
$\eps=1.18$ at order $\order=14$. This translates (via
\autoref{thm:ma_convert}) to
$(2.00, 10^{-5})$-differential privacy, where 2.00 is the expectation of the privacy parameter $\eps$.

These costs are data-dependent and
they cannot be released without further sanitization, which we handle by
adding Gaussian noise scaled by the smooth sensitivity of $\eps$ (the GNSS
mechanism, \autoref{def:GNSS}). At $\beta=0.0329$ the expected value of smooth
sensitivity is $0.0618$. We choose $\sigma_{\SS}=6.23$, which incurs,
according to \autoref{prop:rdp_of_GNSS}, an additional (data-independent)
$(14, 0.52)$-RDP cost. Applying $(\beta,\sigma_{\SS})$-GNSS where
$\sigma_{\SS}=6.23$, we may publish differentially private estimate of the
total privacy cost that consists of a fixed part---the cost of applying
Confident GNMax and GNSS---and random noise. The fixed part is $2.52=1.18+0.52-\ln(10^{-5})/14$, and the noise is  normally distributed with mean 0 and standard deviation $\sigma_{\SS}\cdot 0.0618=0.385$. We note that, in contrast with the standard additive noise, one cannot publish its standard deviation without going through additional privacy analysis.

Some of these constants were optimally chosen (via grid search or analytically) given full view of data, and thus provide a somewhat optimistic view of how this pipeline might perform in practice. For example, $\sigma_{\SS}$ in \autoref{table:smooth_sensitivity} were selected to minimize the total privacy cost plus two standard deviation of the noise.

The following rules of thumb may replace these laborious and privacy-revealing
tuning procedures in typical use cases. The privacy parameter $\delta$ must be
less than the inverse of the number of training examples. Giving a target
$\eps$, the order \order\ can be chosen so that $\log(1/\delta)\approx
(\order-1)\eps/2$, i.e., the cost of the $\delta$ contribution in
\autoref{thm:ma_convert} be roughly half of the total. The $\beta$-smoothness
parameter can be set to $0.4/\order$, from which smooth sensitivity
$\SS_\beta$ can be estimated. The final parameter $\sigma_{\SS}$ can be
reasonably chosen between $2 \cdot \sqrt{(\order+1)/\eps}$ and $4 \cdot \sqrt{(\order+1)/\eps}$ (ensuring that the first, dominant component, of the cost of the GNSS mechanism given by \autoref{prop:rdp_of_GNSS} is between  $\eps/16$ and  $\eps/4$).

\begin{table}
	\centering
	\small\renewcommand{\arraystretch}{1.2}
	\begin{tabular}{|l|l|c|c|c|c|c|c|c|}
		\hline
		& \textbf{Confident GNMax} & \multicolumn{2}{c|}{\textbf{DP}} &  \multicolumn{4}{c|}{\textbf{Smooth Sensitivity}} & \textbf{Sanitized DP} \\ 
		\textbf{Dataset} & \textbf{parameters} &  $\expt{}{\boldsymbol{\eps}}$ &  $\boldsymbol{\delta}$ & $\boldsymbol{\order}$ & $\boldsymbol{\beta}$ &  $\expt{}{\SS_\beta}$ & $\boldsymbol{\sigma}_{\SS}$ & $\expt{}{\boldsymbol{\eps}}\pm\,$\textbf{noise} \\ \hline \hline
		MNIST & $T\texttt{=}200,\sigma_1\texttt{=}150,\sigma_2\texttt{=}40$ &  2.00 & $10^{-5}$ & 14 & .0329 & .0618& 6.23 & 2.52$\,\pm\,$0.385\\ \hline
		SVHN & $T\texttt{=}300,\sigma_1\texttt{=}200,\sigma_2\texttt{=}40$ & 4.96 & $10^{-6}$ & 7.5 & .0533 & .0717 & 4.88& 5.45$\,\pm\,$0.350 \\ \hline
		Adult & $T\texttt{=}300,\sigma_1\texttt{=}200,\sigma_2\texttt{=}40$ & 1.68 & $10^{-5}$& 15.5 &.0310 & 0.0332 & 7.92 & 2.09$\,\pm\,$0.263\\ \hline
		\multirow{3}{*}{Glyph} & $T\texttt{=}1000,\sigma_1\texttt{=}500,\sigma_2\texttt{=}100$& 2.07 & $10^{-8}$ & 20.5 & .0205 &.0128 &11.9 & 2.29$\,\pm\,$0.152\\ \cline{2-9}
		& \multirow{2}{*}{Two-round interactive} & \multirow{2}{*}{0.837} & \multirow{2}{*}{$10^{-8}$} & \multirow{2}{*}{50} & .009 & .00278 & 26.4 & \multirow{2}{*}{1.00$\,\pm\,$.081}\\ \cline{6-8}
		& & &  & & .008 & .00088 & 38.7 &\\ \hline
	\end{tabular}
\caption{\textbf{Privacy-preserving reporting of privacy costs.} The table
  augments \autoref{table:results_summary} by including smooth sensitivity
  analysis of the total privacy cost. The expectations are taken over the
  student's queries and outcomes of the first step of the Confident GNMax
  aggregator. Order \order, smooth sensitivity parameter $\beta$, $\sigma_\SS$
  are parameters of the GNSS mechanism (\autoref{ss:rdp_and_ss}). The final
  column sums up the data-dependent cost $\eps$, the cost of applying GNSS
  (\autoref{prop:rdp_of_GNSS}), and the standard deviation of Gaussian noise
  calibrated to smooth sensitivity (the product of $\expt{}{\SS_\beta}$ and $\sigma_\SS$).}
\label{table:smooth_sensitivity}
\end{table}

\end{document}